\documentclass[twoside]{article}
\usepackage[utf8]{inputenc} 
\usepackage[T1]{fontenc}    
\usepackage{hyperref}       
\usepackage{url}            
\usepackage{booktabs}       
\usepackage{amsfonts}       
\usepackage{nicefrac}       
\usepackage{microtype}      
\usepackage{wrapfig}
\usepackage{subfigure}
\usepackage{xingstyle}
\usepackage{graphicx}
\usepackage{appendix}
\usepackage{algorithm}
\usepackage{algorithmic}
\usepackage{amsmath}
\usepackage{mathrsfs}
\usepackage[round]{natbib}
\usepackage{tabu}
\usepackage{textcomp}
\usepackage{multirow}
\usepackage{color}
\usepackage{tikz}
\usepackage{hhline}
\usepackage[noeepic]{qtree}
\renewcommand{\L}{\mathcal{L}}
\usepackage[accepted]{aistats2021}
\begin{document}

%

%

\twocolumn[

\aistatstitle{Simultaneously Reconciled Quantile Forecasting of Hierarchically Related Time Series}

\aistatsauthor{ Xing Han \And Sambarta Dasgupta \And  Joydeep Ghosh }

\aistatsaddress{ UT Austin \\ \texttt{aaronhan223@utexas.edu} \And  IntuitAI \\ \texttt{dasgupta.sambarta@gmail.com} \And UT Austin \\ \texttt{jghosh@utexas.edu}} 

]

\begin{abstract}
Many real-life applications involve simultaneously forecasting multiple time series that are hierarchically related via aggregation or disaggregation operations. For instance, commercial organizations often want to forecast inventories simultaneously at store, city, and state levels for resource planning purposes. In such applications, it is important that the forecasts, in addition to being reasonably accurate, are also consistent w.r.t one another. Although forecasting such hierarchical time series has been pursued by economists and data scientists, the current state-of-the-art models use strong assumptions, e.g., all forecasts being unbiased estimates, noise distribution being Gaussian. Besides, state-of-the-art models have not harnessed the power of modern nonlinear models, especially ones based on deep learning. In this paper, we propose using a flexible nonlinear model that optimizes quantile regression loss coupled with suitable regularization terms to maintain the consistency of forecasts across hierarchies. The theoretical framework introduced herein can be applied to any forecasting model with an underlying differentiable loss function. A  proof of optimality of our proposed method is also provided. Simulation studies over a range of datasets highlight the efficacy of our approach.
\end{abstract}

\section{Introduction} \label{background}
\label{sec:intro}

\newcommand{\hy}{\hat{y}}
\newcommand{\hb}{\hat{\beta}}
\newcommand{\ty}{\Tilde{y}}
\newcommand{\hY}{\hat{Y}}
\newcommand{\tY}{\Tilde{Y}}
\newcommand{\TL}{\widetilde{\mathcal{L}}_n}
\renewcommand{\S}{\mathcal{S}}
\newcommand{\I}{\mathcal{I}}

\tikzstyle{circ} = [draw, circle, fill=white!20, radius=2.6, minimum size=.8cm, inner sep=0pt]
\tikzstyle{line} = [draw]

\begin{figure*}
\centering
{
\setlength{\tabcolsep}{1pt} 
\renewcommand{\arraystretch}{1} 
    \begin{tabu*}{cccc}
    \includegraphics[width=.26\textwidth]{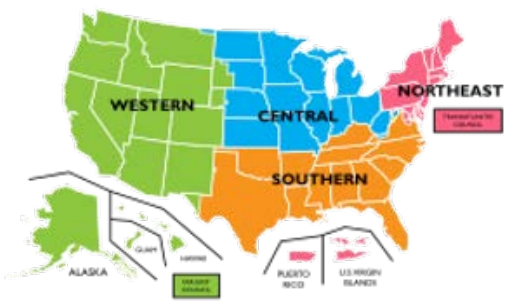} & 
    \begin{tikzpicture}[node distance = 1.cm,auto]
    \node [circ] (step1) {$v_1$};
    \node [circ, below left = 0.3cm and 1.2cm of step1, label={793:{$e_{1, 2}$}}] (step2) {$v_2$};
    \node [circ, below right = 0.3cm and 1.2cm of step1, label={493:{$e_{1, 3}$}}] (step3) {$v_3$};
    \node [circ, below left = 0.3cm and 0.3cm of step2] (step4) {$v_4$};
    \node [circ, below right = 0.3cm and 0.3cm of step2] (step5) {$v_5$};
    \node [circ, below left = 0.3cm and 0.3cm of step3] (step6) {$v_6$};
    \node [circ, below right = 0.3cm and 0.3cm of step3] (step7) {$v_7$};
    \path [line, rounded corners] (step1) -| (step2);
    \path [line, rounded corners] (step1) -| (step3);
    \path [line, rounded corners] (step2) -| (step4);
    \path [line, rounded corners] (step2) -| (step5);
    \path [line, rounded corners] (step3) -| (step6);
    \path [line, rounded corners] (step3) -| (step7);
    \end{tikzpicture} &
    \includegraphics[width=.18\textwidth]{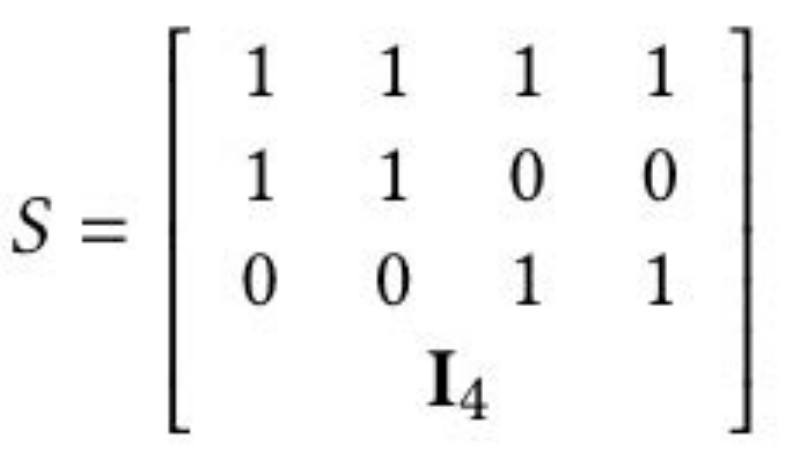} &
    \includegraphics[width=.18\textwidth]{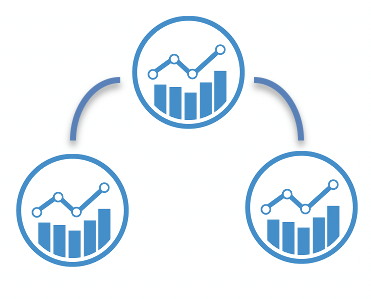} \\
    (a) & (b) & (c) & (d) \\
    \end{tabu*}
    \caption{(a) Example of hierarchically related time series: state population growth forecast, the data is aggregated by geographical locations; (b) corresponding graph structure as nodes and vertices; (c) corresponding matrix representation with four bottom level time series and three aggregated levels; (d) time series forecast at each node.}}
    \label{fig:hierarchical_demo}
\end{figure*}
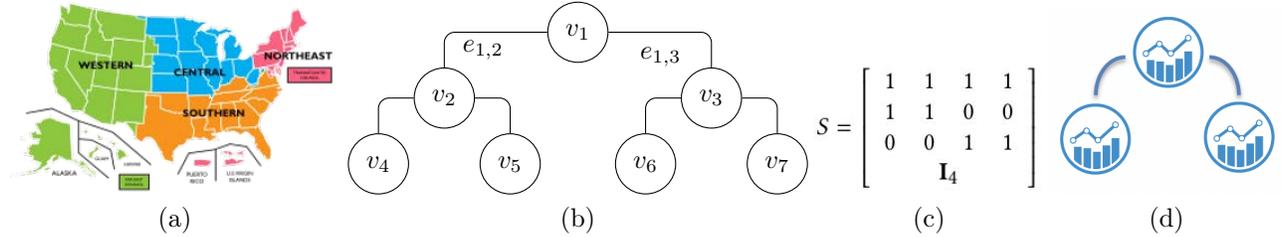

Hierarchical time series refers to a set of time series organized in a logical hierarchy with the parent-children relations, governed by a set of aggregation and disaggregation operations \citep{hyndman2011optimal, taieb2017coherent}. These aggregations and disaggregations can occur across multiple time series or over the same time series across multiple time granularities. An example of the first kind can be forecasting demand at county, city, state, and country levels \citep{hyndman2011optimal}. An example of the second kind of hierarchy is forecasting demand at different time granularities like daily, weekly, and monthly \citep{athanasopoulos2017forecasting}. The need to forecast multiple time series that are hierarchically related arise in many applications, from financial forecasting \citep{sasforecasting} to demand forecasting \citep{hyndman2016fast, zhao2016multi} and psephology \citep{lauderdale2019model}. The
recently announced M5 competition\footnote{https://mofc.unic.ac.cy/m5-competition/} from the International Institute of Forecasters also involves hierarchical forecasting on Walmart data with a \$100K prize. 
A novel challenge in such forecasting problems is to produce accurate forecasts while maintaining the consistency of the forecasts across multiple hierarchies.

\paragraph{Related Works} Existing hierarchical forecasting methods predominantly employ linear auto-regressive (AR) models that are initially trained while ignoring the hierarchy. The output forecasts produced by these AR models are reconciled afterward for consistency. As shown in Figure \ref{fig:hierarchical_demo}(c), such reconciliation is achieved by defining a mapping matrix, often denoted by $S$, which encapsulates the mutual relationships among the time series \citep{hyndman2011optimal}. The reconciliation step involves inversion and multiplication of the $S$ matrix that leads to the computational complexity of  $\mathcal{O}(n^3 h)$, where $n$ is the number of nodes, and $h$ represents how many levels of hierarchy the set of time series are organized into. Thus,  reconciling hundreds of thousands of time series at a time required for specific industrial applications becomes difficult. Improved versions of the reconciliation were proposed by employing trace minimization \citep{wickramasuriya2015forecasting} and sparse matrix computation \citep{hyndman2016fast}. These algorithms assume that the individual base estimators are unbiased, which is unrealistic in many real-life applications. The unbiasedness assumption was relaxed in \citep{taieb2017coherent} while introducing other assumptions like ergodicity with exponentially decaying mixing coefficient. Moreover, all these existing methods try to impose any reconciliation constraints among the time series as a post inference step, which possesses two challenges: 1. ignoring the relationships across time series during training can potentially lead to suboptimal solutions, 2. additional computational complexity in the inference step, owing to the inversion of $S$ matrix, which makes it challenging to use the technique when a significant amount of time series are encountered in an industrial forecasting pipeline. 

A critical development in time series forecasting has been the application of Deep Neural Networks (DNN) \citep[e.g.,][]{chung2014empirical, lai2018modeling, mukherjee2018armdn, oreshkin2019n, salinas2019deepar, sen2019think, zhu2017deep} which have shown to outperform statistical auto-regressive or other statistical models in several situations. However, no existing method incorporates the hierarchical structure of the set of time series into the DNN learning step. Instead, one hopes that the DNN will learn the relationships from the data. Graph neural networks (GNN) \citep[e.g.,][]{franceschi2019learning, lachapelle2019gradient, wu2020connecting, yu2017spatio, yu2019dag, zhang2020deep, zheng2018dags} have also been used to learn inherent relations among multiple time series; however, they need a pre-defined graph model that can adequately capture the relationships among the time series.
Characterizing uncertainty of the DNN forecast is another critical aspect, which becomes even more complicated when additive noise does not follow a Gaussian distribution \citep[e.g.,][]{blundell2015weight, iwata2017improving, kuleshov2018accurate, lakshminarayanan2017simple, sun2019functional}. If the observation noise model is misspecified, then the performance would be poor however complex neural network architecture one uses. Other works like \citet{salinas2019deepar} use multiple observation noise models (Gaussian, Negative Binomial) and loss functions. It is left to human experts' discretion to select the appropriate loss based on the time series's nature. This approach cannot be generalized and involves human intervention; it is especially not feasible for an industrial forecasting pipeline where predictions are to be generated for a vast number of time series, which can have a widely varying nature of observation noise. Besides, Bayesian approaches also face problems as the prior distribution and loss function assumptions will not be met across all of the time series. One approach that does not need to specify a parametric form of distribution is through quantile regression. Prior works include combining sequence to sequence models with quantile loss to generate multi-step probabilistic forecasts \citep{wen2017multi} and modeling conditional quantile functions using regression splines \citep{gasthaus2019probabilistic}. These works are incapable of handling hierarchical structures within time series. 

Key aspects of multiple, related time series forecasting addressed by our proposed model include:
\begin{enumerate}
	\item introduction of a regularized loss function that captures the mutual relationships among each group of time series from adjacent aggregation levels,
	\item generation of probabilistic forecasts using quantile regression and simultaneously reconciling each quantile during model training,
	\item clear demonstration of superior model capabilities, especially on real e-commerce datasets with sparsity and skewed noise distributions.
\end{enumerate}

\paragraph{Background: Hierarchical Time Series Forecast}  Denote $b_t \in \mathbb{R}^m$, $a_t \in \mathbb{R}^k$ as the observations at time $t$ for the $m$ and $k$ series at the bottom and aggregation level(s), respectively. Then $y_t = S b_t \in \mathbb{R}^n$ contains observations at time $t$ for all levels. Similarly, let $\hat{b}_T(h)$ be the $h-$step ahead forecast on the bottom-level at time $T$, we can obtain forecasts in higher aggregation levels by computing $\hy_T(h) = S \hat{b}_T(h)$. This simple method is called bottom-up (BU), which guarantees reconciled forecasts. However, the error from bottom-level forecasts will accumulate to higher levels, leading to poor results. BU also cannot leverage any training data that is available at the more granular levels. A more straightforward approach called base forecast is to perform forecasting for each time series independently without considering the structure at all, i.e., compute $\hy_T(h) = \left[\hat{a}_T (h)^{\top} ~~ \hat{b}_T(h)^{\top}\right]^{\top}$, but this will apparently lead to irreconciled forecasts. Therefore, imposing constraints to revise the base forecasts is a natural choice to accommodate the hierarchical relationships. More specifically, the goal is to obtain some appropriately selected matrix $P \in \mathbb{R}^{m \times n}$ to combine the base forecasts linearly: $\ty_T(h) = SP\hy_T(h)$, where $\ty_T(h)$ is the reconciled forecasts which are now coherent by construction. The role of $P$ is to map the base forecasts into the forecasts at the most disaggregated level and sum them up by $S$ to get the reconciled forecasts. The previously mentioned approach \citep{ben2019regularized, hyndman2011optimal, hyndman2016fast, wickramasuriya2015forecasting} involves computing the optimal $P$ under different situations. A more detailed introduction can be found in Appendix \ref{sec:background}.
\section{Simultaneous Hierarchically Reconciled Quantile Forecasting} \label{dqhfn}
We propose a new method for hierarchical time-series forecasts. Our approach fundamentally differs from others in that we move the reconciliation into the training stage by enabling the model to simultaneously learn the time series data from adjacent aggregation levels, while also integrating quantile regression to provide coherent forecasts for uncertainty bounds. We call our method Simultaneous HierArchically Reconciled Quantile Regression (SHARQ) to highlight these properties.

\subsection{Problem Formulation} \label{problem_formulation}

\paragraph{Graph Structure} Figure \ref{fig:hierarchical_demo}(b) shows a hierarchical graph structure where each node represents a time series, which is to be predicted over a horizon. The graph structure is represented by $ \{ V, E \}$, where $V := \{v_1, v_2, \dots, v_n \} $ are the vertices of the graph and $E := \{ e_{i_1, j_1}, e_{i_2, j_2}, \dots, e_{i_p, j_p} \}$ are the set of edges. Also, $e_{i, j} \in \{-1, 1\}$ is a signed edge where $i$ is the parent vertex and $j$ is the child. An example of the negative-signed edge can be forecasting surplus production, which is the difference between production and demand. The time series for vertex $i$ at time $t$ is represented as $x_{v_i} (t)$. For sake of simplicity, we assume the value of a time series at a parent level will be the (signed)  sum of the children vertices. 
This constraint relationship can be represented as $x_{v_i} (t) = \sum_{e_{i, k} \in E} e_{i, k} ~ x_{v_k} (t).$ We can later extend these to set of non-linear constraints: $x_{v_i} (t) =H_{v_i} \left(  \sum_{e_{i, k} \in E} e_{i, k} ~ x_{v_k} (t) \right)$, where $ H_{v_i} \in \mathbb{C}^1$. But linear hierarchical aggregation has already covered most real-world applications. The graph has hierarchies $L := \{ l_1, l_2, \dots, l_i, \dots, l_q\}$, where $l_i$ is the set of all vertices belonging to the $i^{th}$ level of hierarchy. Note that the graph representation using $\{ V, E \}$ is equivalent to the $S$ matrix in defining a hierarchical structure.

\tikzstyle{circ} = [draw, circle, fill=white!20, radius=2.6, minimum size=.8cm, inner sep=0pt]
\tikzstyle{line} = [draw]


\paragraph{Data Fit Loss Function}
We now formulate the learning problem for each node. Let $\{ x_{v_i} (t) ~|~ t = 0, \dots, T \}$ be the training data for vertex $v_i$, $w$ be the window of the auto-regressive features, and $h$ be horizon of forecast. Based on the window, we can create a sample of training data in the time stamp $m$ as: 
\begin{align*}
    \{ (X_m^i, Y_m^i) ~|~ X_m^i & = [ F \left ( x_{v_i} (m), \dots, x_{v_i} (m - w + 1) \right ) ], \\
    Y_m^i & = [x_{v_i} (m + 1), \dots, x_{v_i} (m + h)]  \}, \notag
\end{align*}
where $F \colon \mathbb{R}^{w+1} \to \mathbb{R}^{n_f}$ is the featurization function, which will generate a set of features; $n_f$ is the size of feature space, $h$ is the horizon of the forecast. It can be noted that $n_f \geq \omega$. In this fashion, we transform the forecasting problem to a regression one, where $n_f$ and $h$ capture the size of the feature space and the response. For instance, we can create a set of features for an ARIMA$(p, d, q)$ model based on the standard parameterization, where auto-regressive window size $w=p$ and the other features corresponding to the $d, q$ parameters will form the rest of the features.

We represent a forecasting model that learns the mean as a point estimate. Denote function $g_i : \mathbb{R}^{n_f} \times \mathbb{R}^{n_\theta} \to \mathbb{R}^{d}$, where $n_\theta$ represents the number of model parameters which are represented as $\theta_i \in \mathbb{R}^{n_\theta}$. The estimate from the model for $X_m^i$ will be $\hat{Y}_m^i := g_i (X_m^i, \theta_i)$. We can define a loss function for the $m^{th}$ sample at the $i^{th}$ vertex as $\L (\hat{Y}_m^i, ~ Y_m^i) = \L (g_i (X_m^i, \theta_i), ~ Y_m^i)$. We would assume the noise in training samples are of \texttt{i.i.d.} nature. As a consequence, the loss for the entire training data will be sum for each sample, i.e., $\sum_m \L (g_i (X_m^i, \theta_i), Y_m^i)$. Noted that this formulation will work for neural networks or ARIMA but not for Gaussian Processes, where we need to model the covariance of the uncertainties across the samples.

\paragraph{Reconciled Point Forecast}
We then describe how to incorporate reconciliation constraints into the loss functions for the vertices. The constraints at different levels of hierarchy will have different weights, which we denote as a function $w_c: L \to \mathbb{R}$. We define another function that maps any vertex to the hierarchy level, $LM: V \to L$. For any vertex $v_i$, the corresponding weight for the constraint is given by  $ \lambda_{v_i} := w_c \circ LM (v_i)$. The constrained loss for vertex $v_i$ will be 
\begin{align} 
& \L_c (g_i (X_m^i, \theta_i), Y_m^i, g_k (X_m^k, \theta_k)) := \L (g_i (X_m^i, \theta_i), Y_m^i) \nonumber \\
& + \lambda_i \parallel g_i (X_m^i, \theta_i) - \sum_{e_{i, k} \in E} \left ( e_{i, k} ~g_k (X_m^k, \theta_k) \right ) \parallel^2 .
\label{contstrained_loss}
\end{align}  
Note that the data fit loss and the reconciliation, as described thus far, are catered to the point estimate of the mean forecasts.

\subsection{Reconciling Probabilistic Forecast using Quantiles} \label{sec_qloss}
Generating probabilistic forecasts over a range of time is significant for wide-ranging applications. Real-world time series data is usually sparse and not uniformly sampled. It is unreasonable to assume that the uncertainty or error distribution at every future point of time as Gaussians. A standard approach to solve this problem is using quantile loss, which allows one to model the error distribution in a non-parametric fashion. The estimator will aim at minimizing a loss directly represented in terms of the quantiles. Simultaneously, quantile can be used to construct confidence intervals by fitting multiple quantile regressors to obtain estimates of upper and lower bounds for prediction intervals. The quantile loss $\rho_{\tau} (y)$ is defined as
$\rho_{\tau} (y, ~ Q_{\tau}) = \left( y - Q_{\tau} \right). (\tau - \mathbb{I}_{(y<Q_{\tau})}),$
where $Q_{\tau}$ is the $\tau^{th}$ quantile output by an estimator. We will adopt quantiles in our framework, and pre-specified quantile ranges will represent the forecasting distribution. For simplicity, we denote $Q_i^{\tau} = Q^{\tau}_i (X_m^i, \theta_i)$ as the $\tau^{th}$ quantile estimator for node $i$. Eq.(\ref{quantile}) demonstrates a quantile version of the probabilistic estimator; it aims to produce multiple forecasts at different quantiles while maintaining a coherent median forecast which will be used to reconcile other quantiles:
\begin{align}
     & \L_c (Q_i, Y_m^i, Q_k) := \nonumber \\
     & \sum_{\tau = \tau_0}^{\tau_q}\rho_{\tau} (Q^{\tau}_i, Y_m^i) 
 + \lambda_i  \parallel Q^{50}_i - \sum_{e_{i, k} \in E} e_{i, k}  Q^{50}_i \parallel^2,
\label{quantile}
\end{align}
where $[\tau_0, \dots, \tau_q]$ are a set of quantile levels, and $Q_i = [Q^{\tau_0}_i, \dots, Q^{\tau_q}_i]$. To further guarantee that estimation at each quantile is coherent across the hierarchy, the straightforward solution is to add consistency regularization for each quantile like Eq.(\ref{quantile}). However, too many regularization terms would not only increase the number of hyper-parameters, but also complicate the loss function, where the result may even be worse as it is hard to find a solution to balance each objective. Moreover, a quantile estimator does not hold the additive property. As an example, assume $X_1 \sim N(\mu_1, \sigma_1^2)$, $X_2 \sim N(\mu_2, \sigma_2^2)$ are independent random variables, and define $Y = X_1 + X_2$. Then $Q_{Y}^{\tau} = Q_{X_1}^{\tau} + Q_{X_2}^{\tau}$ is true only if $X_1 = C \times X_2$, where $C$ is arbitrary constant. This requirement cannot be satisfied. But for any $\tau$, we have $(Q_{Y}^{\tau} - \mu_Y )^2 = (Q_{X_1}^{\tau} - \mu_{X_1})^2 + (Q_{X_2}^{\tau} - \mu_{X_2})^2$, the proof of above properties can be found in Appendix \ref{sec:nonadd}. Given these properties, we can formulate a new objective to make an arbitrary quantile consistent:
\begin{align}
    & \L_q (Q_i, Y_m^i, Q_k) := \label{quantile_recon} \\
    & \left [f\left(Q^{\tau}_i - Q^{50}_i\right) 
 	 - \sum_{e_{i, k} \in E} e_{i, k} ~ f\left(Q^{\tau}_k - Q^{50}_k\right) + \mathrm{Var}(\epsilon)\right ]^2, \nonumber
\end{align}
where $\epsilon$ is the mean forecast's inconsistency error, which is mostly a much smaller term for a non-sparse dataset, and $f$ is the distance metric. Eq.(\ref{quantile_recon}) is zero when $f$ is a squared function and the given data satisfies \texttt{i.i.d.} Gaussian assumption. Therefore,  optimizing Eq.(\ref{quantile_recon}) ``forces'' this additive property in non-Gaussian cases and is equivalent to reconcile the quantile estimators, which can also be interpreted as reconciliation over the variance across adjacent aggregation levels. Empirically, this approach calibrates multiple quantile predictions to be coherent and mitigates the quantile crossing issue \citep{liu2009stepwise}.

\begin{algorithm}[t]
\small
\caption{SHARQ}
\label{alg:bu}
\begin{algorithmic}
\STATE \textbf{Input}:~~Training data $\mathcal{I}_1 = \{X_i, Y_i\}_{i=1}^{T_1}$, testing data $\mathcal{I}_2 = \{X_i, Y_i\}_{i=1}^{T_2}$.
\STATE \textbf{Process}:
\STATE train each leaf node (e.g., $v_4$ to $v_7$ in Figure \ref{fig:hierarchical_demo}(b)) independently without regularization
\FOR{each vertex $v$ at upper level $l$ (e.g., $l = 2$ for $v_3$, $v_2$, then $l = 1$ for $v_1$)}
\STATE train vertex $v$ at level $l$ using Eq.(\ref{quantile})
\ENDFOR
\STATE Reconciled Median Forecast \textbf{MF} $\leftarrow$ \textbf{Models}($\mathcal{I}_2$)
\FOR{each vertex $v$ at upper level $l$}
\STATE train vertex $v$ at level $l$ using Eq.(\ref{quantile_recon}) and \textbf{MF}
\ENDFOR
\STATE \textbf{Output}: ~ Reconciled forecasts at pre-specified quantiles.
\end{algorithmic}
\end{algorithm}

\subsection{SHARQ Algorithm}
Our formulation can be combined with a bottom-up training approach to reconciling forecasts for each quantile simultaneously. Since the time series at the leaf (disaggregated) level are independent of higher aggregation levels, we can use the lower-level forecasting results to progressively reconcile the forecasting models at higher levels without revisiting previous reconciliations, till the root is reached. In contrast, if top-down training is applied, one needs to reconcile both higher (previously visited) and lower-level data at an intermediate vertex, since other time series at that intermediate level may have changed. Algorithm \ref{alg:bu} describes our procedure. We now address the remaining aspects.

\paragraph{Beyond Gaussian Noise Assumption.} Noise distributions for many real-world datasets (e.g., e-commerce data) are heavily skewed, so a Gaussian model may not be appropriate. For multi-level time series, data at the most disaggregated (lowest) level is more likely to be sparse, which makes the quantile reconciliation in Eq.(\ref{quantile_recon}) less accurate. In such situations, one can substitute median with the mean estimator as in Eq.(\ref{contstrained_loss}) for training lower-level time series. One can also mix-and-match between mean and median estimators at higher aggregation levels depending on the data characteristics. Finding a suitable function $f$ for quantile reconciliation as in Eq.(\ref{quantile_recon}) is an alternative way to tackle  non-symmetric errors \citep{li20081}. 
    
\paragraph{Efficient Training and Inference} SHARQ is time and memory efficient, scaling well in both aspects with large datasets. One can simultaneously train multiple time series and keep a running sum for reconciliation. Since coherent probabilistic forecasts are enforced during training, no extra post-processing time is needed (see Appendix \ref{sec:add_exp} for details). Besides, SHARQ does not force one to use deep forecasting models or to use the same type of model at each node; in fact, any model where gradient-based optimization can be used is allowable, and one can also mix-and-match. For cases where the time series at a given level are structurally very similar \citep{zhu2017deep}, they can be grouped (e.g., by clustering), and a single model can be learned for the entire group.




\newcommand\Item[1][]{%
  \ifx\relax#1\relax  \item \else \item[#1] \fi
  \abovedisplayskip=0pt\abovedisplayshortskip=0pt~\vspace*{-\baselineskip}}
  
\section{Statistical Analysis of SHARQ} \label{theory}
In this section, we theoretically demonstrate the advantages of SHARQ. We begin by showing the optimality of our formulation (\ref{contstrained_loss}) in contrast to post-inference based methods, which solves the matrix $P$ for $\ty_T(h) = SP\hy_T(h)$. We emphasize our advantage in balancing coherency requirements and forecasting accuracy. We then present some desirable statistical properties of SHARQ. 

\begin{thm}\textbf{(Global Optimum)} \label{thm:loss}
For $L_c \in C^1$, for an arbitrary parameterized smooth regressor model asymptotically,
\begin{align} 
&\L_c (g_i (X_m^i, \theta^{\star}_i), ~ Y_m^i, ~~ g_k (X_m^k, \theta^{\star}_k)) \nonumber \\
&\le \L_c (g_i (X_m^i, \theta^{recon}_i), ~ Y_m^i, ~~ g_k (X_m^k, \theta^{recon}_k))
\end{align}
where $\theta^{\star}$, $\theta^{recon}$ are the parameters for SHARQ, and post inference reconciled solution, respectively. 
\end{thm}

\begin{proof}
By definition, SHARQ directly minimize Eq.(\ref{contstrained_loss}),
\begin{equation}
    \theta^{\star} = \arg \underset{\theta}{\min} ~\L_c (g_i (X_m^i, \theta_i), ~ Y_m^i, ~~ g_k (X_m^k, \theta_k)),
\end{equation}
where $\theta = \{\theta_i, \theta_k\}$. 
\end{proof}
Since \citet{ben2019regularized, hyndman2011optimal, hyndman2016fast, wickramasuriya2015forecasting} are performing the reconciliation as a post-processing step, those solutions are bound to be sub-optimal in comparison with $\theta^{\star}$.
\begin{pro}\textbf{(Hard Constraint)} \label{thm:stats}
For post-inference based methods, $P \hat{y}_T (h)$ is the bottom level reconciled forecast. In other words, it requires that 
\begin{equation}
    g_i (X_m^i, \theta^{\star}_i) = \sum_{e_{i, k} \in E} g_k (X_m^k, \theta^{recon}_k).
    \label{equ:stats_pro}
\end{equation}
Had we only considered the point forecast reconciliation, i.e. $\E[g_i (X_m^i, \theta^{\star}_i)] = \sum_{e_{i, k} \in E} \E[g_k (X_m^k, \theta^{recon}_k)],$ the post inference processing still might have worked. However, due to the probabilistic nature of the variables $g_i (X_m^i, \theta^{\star}_i) = \E[g_i (X_m^i, \theta^{\star}_i)] + \varepsilon_i$, where $\varepsilon_i$ is the observation noise, reconciling the mean won't suffice.
\end{pro}
\paragraph{Remark}  To satisfy Eq.(\ref{equ:stats_pro}), it is required that $\varepsilon_i = \sum_{e_{i, k} \in E} \varepsilon_k$, which is not a realistic property to be satisfied in real-life problems. Intuitively, when a reconciliation matrix $P$ is applied, the original, unbiased base forecasts with variation are ``forced'' to be summed up. However, our method does not impose this hard constraint, leading to different properties.

\begin{figure*}[t]
\centering
{
\setlength{\tabcolsep}{1pt} 
\renewcommand{\arraystretch}{1} 
\begin{tabu}{ccc}
\hspace{-0.45em}
\includegraphics[width=.33\textwidth]{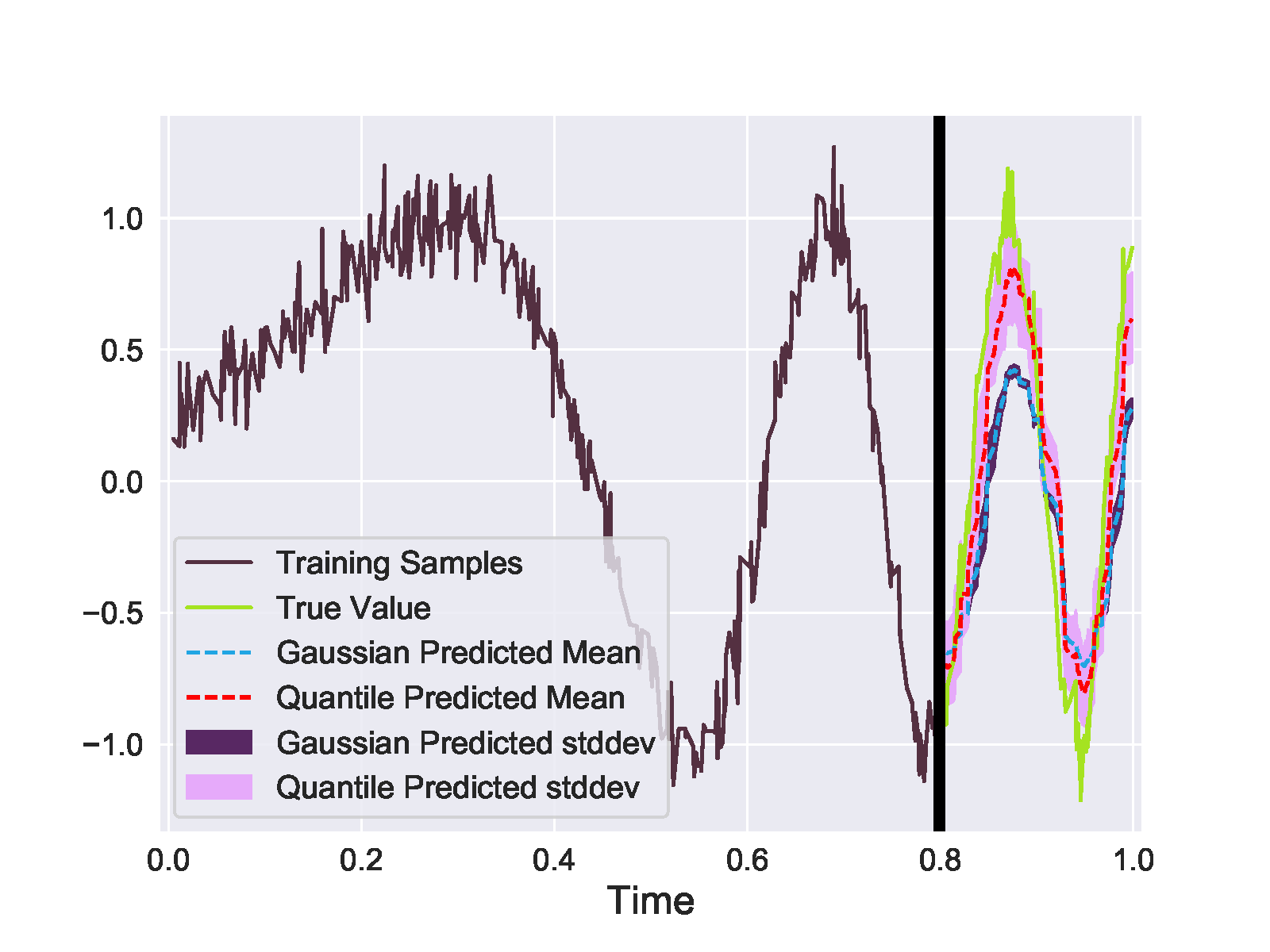}&
\hspace{-0.45em}
\includegraphics[width=.33\textwidth]{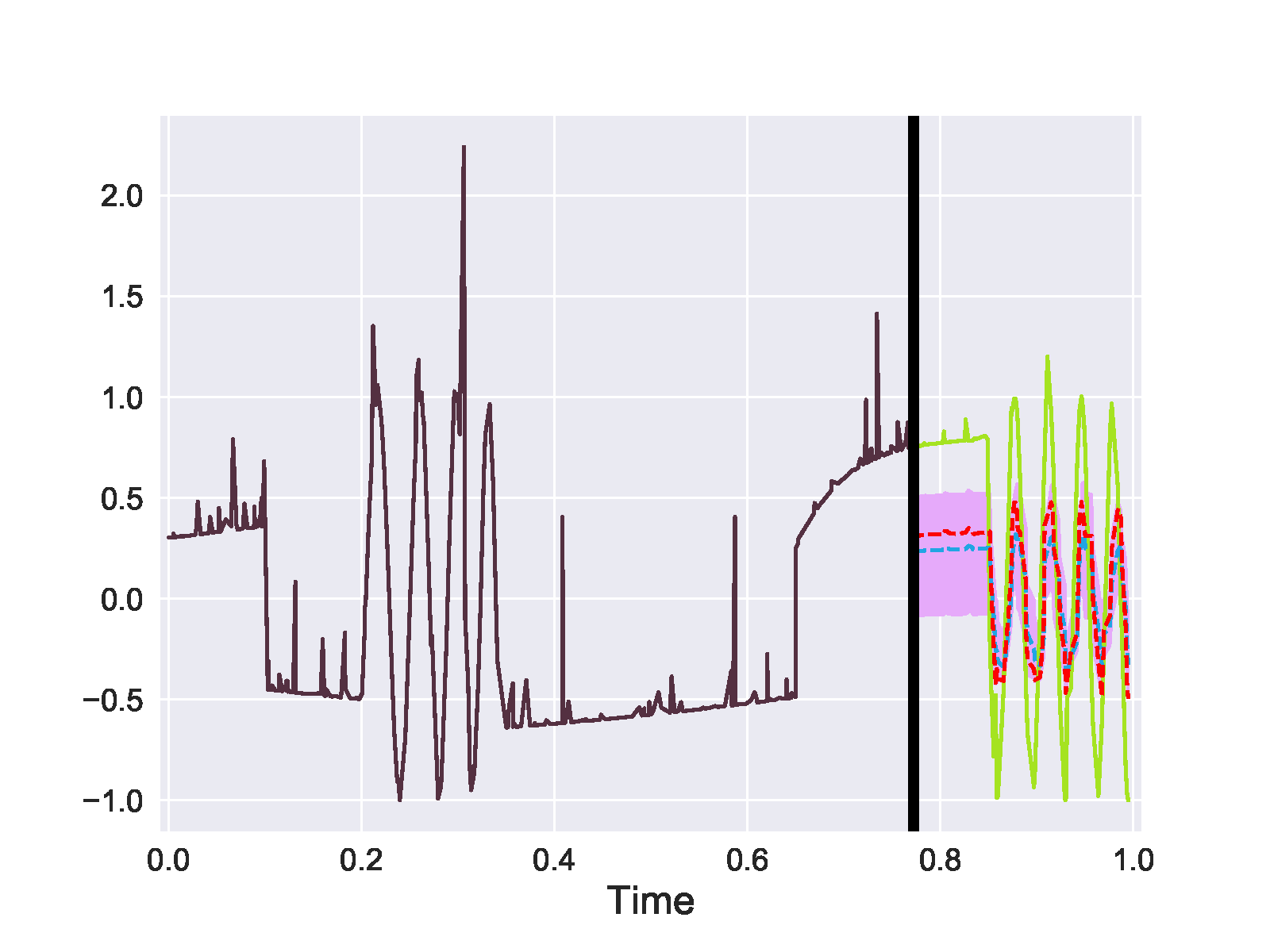}&
\hspace{-0.45em}
\includegraphics[width=.33\textwidth]{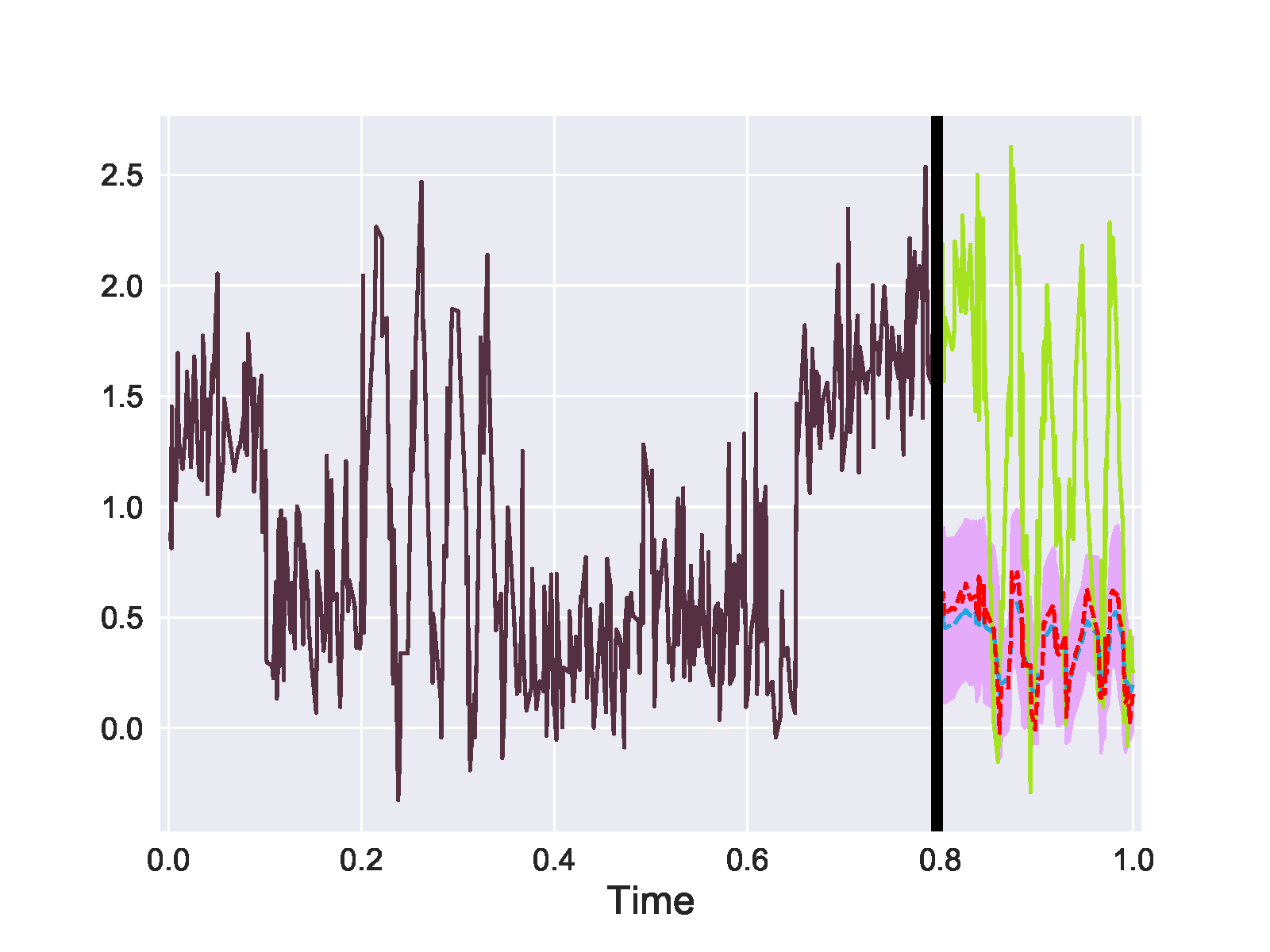} \\
\hspace{-0.45em} (a) Log-normal error & \hspace{-0.45em} (b) Gamma error & \hspace{-0.45em} (c) Gaussian error \\
\end{tabu}
}
\caption{Forecasting results of simulated sequential data under different error distributions. The original function of (a) is sinusoidal with varying frequency; (b) and (c) are discontinuous step functions. Note that our baseline forecasts in (b) and (c) are overconfident \citep{lakshminarayanan2017simple, li2020improving} and their prediction intervals are too small to be shown.}
\label{fig:sim_plot}
\end{figure*}

\usetikzlibrary{fit,backgrounds} 
\usetikzlibrary{shadows.blur}
\usetikzlibrary{shapes,arrows}

\tikzstyle{rect} = [draw, rectangle, fill=white!20, text width=2em, text centered, minimum height=2em]
\tikzstyle{longrect} = [draw, rectangle, fill=white!20, text width=6em, text centered, minimum height=2em]
\tikzstyle{circ} = [draw, circle, fill=white!20, radius=2.6, minimum size=.8cm, inner sep=0pt]
\tikzstyle{line} = [draw, -latex']

\begin{table*}[t]
\centering
\renewcommand\arraystretch{1.7}
\begin{minipage}{.41\textwidth}
\scalebox{0.77}{
\begin{tabular}{c|c|c|c|c}
\hlinewd{1.3pt}
\multicolumn{2}{c}{ Error distribution } & Log-normal & Gamma & Gaussian \\ \hline
\multirow{2}{*}{ MAPE } & Quantile & \textbf{32.72} & \textbf{62.20} & \textbf{70.29} \\
& Baseline & 49.88 & 73.69 & 73.36 \\ \hline
\multirow{2}{*}{ LR } & Quantile & \textbf{0.2218} & \textbf{0.5649} & \textbf{0.8634} \\
& Baseline & 0.5661 & 0.7489 & 1.025 \\
\hlinewd{1.3pt}
\end{tabular}}
\end{minipage}
\begin{minipage}{.57\textwidth}
\centering
    \begin{tikzpicture}[node distance = 1.2cm,auto]
    \node [rect, rounded corners, label={[font=\small\sffamily,name=label1]above:{Forecasting model}}] (step1) {$h_{t-2}$};
    \node [rect, right of=step1, rounded corners] (step2) {$h_{t-1}$};
    \node [rect, right of=step2, rounded corners] (step4) {$h_t$};
    \node [rect, right of=step4, rounded corners] (step8) {$o_t$};
    \node [circ, below of=step1] (step5) {$x_{t-2}$};
    \node [circ, below of=step2] (step6) {$x_{t-1}$};
    \node [circ, below of=step4] (step7) {$x_t$};
    \node [longrect, above right = 0.2cm and 0.3cm of step8] (step9) {dense layer 1};
    \node [longrect, below right = 0.2cm and 0.3cm of step8] (step10) {dense layer 3};
    \node [longrect, right =0.25cm of step8] (step12) {dense layer 2};
    \node [circ, right =0.2cm of step12] (step11) {$\hat{y}^{\tau}_{t+1}$};
    \path [line] (step1) -- (step2);
    \path [line] (step2) -- (step4);
    \path [line] (step5) -- (step1);
    \path [line] (step6) -- (step2);
    \path [line] (step7) -- (step4);
    \path [line] (step4) -- (step8);
    \path [line, rounded corners] (step8) |- (step9);
    \path [line, rounded corners] (step8) |- (step10);
    \path [line, rounded corners] (step9) -| (step11);
    \path [line, rounded corners] (step10) -| (step11);
    \path [line] (step8) -- (step12);
    \path [line] (step12) -- (step11);
    \begin{scope}[on background layer]
    \node[draw,dashed,black,rounded corners,fill=gray!50,fit=(step1) (step2) (step4) (step5) (step6) (step7) (label1)]{};
    \end{scope}
    \end{tikzpicture}
\end{minipage}
\caption{Left: Quantitative results for simulation experiments. We use Mean Absolute Percentage Error (MAPE) \citep{makridakis2000m3} to measure forecasting accuracy and Likelihood Ratio (LR) to evaluate uncertainty intervals. Right: A schematic of the multi-quantile forecaster used for simulation.}
\label{tab:sim_res}
\end{table*}

\begin{pro}\textbf{(Unbiasedness Property)}
Consider a hierarchical structure with $n$ nodes where the first $\kappa$ nodes belong to aggregation levels, assume that the bottom level forecasting is unbiased:
\begin{equation}
    \E [g_k (X_m^k, \theta_k)] = Y_m^k, \quad k = \kappa + 1,..., n
\end{equation}
and the bottom level forecasting models are well optimized:
\begin{equation}
     \mathrm{Var}\left(\|g_k (X_m^k, \theta_k) - Y_m^k\|\right) = \epsilon, \quad
     \epsilon = \mathcal{O}(\frac{1}{m}).
\end{equation}
Then we have that
\begin{align*}
    \E[g_i (X_m^i, \theta_i)] & = \E[\sum_{e_{i, k} \in E} g_k (X_m^k, \theta_k)] \nonumber \\
    & = \sum_{e_{i, k} \in E} \E[g_k (X_m^k, \theta_k)] \nonumber \\
    & = Y_m^i, \quad i = 1, ..., \kappa
\end{align*}
\end{pro}
Therefore, we claim that given the unbiased base forecast at the most disaggregated level, as well as well-specified models, our method can provide unbiased estimation at all aggregation levels.
\begin{pro}\textbf{(Variance Reduction)}
Assume we are minimizing a quadratic loss function using our formulation, where
\begin{align*}
& L_c (g_i (X_m^i, \theta_i), Y_m^i, g_k (X_m^k, \theta_k)) = \|g_i (X_m^i, \theta_i) - Y_m^i\|^2 \nonumber \\
& + \lambda_i  \parallel g_i (X_m^i, \theta_i) - \sum_{e_{i, k} \in E} \left ( e_{i, k} ~g_k (X_m^k, \theta_k)  \right ) \parallel^2.
\end{align*}
By solving the quadratic objective, we get
\begin{equation}
    g_i (X_m^i, \theta_i) = \frac{Y_m^i + \lambda_i \sum_{e_{i, k} \in E} e_{i, k} ~g_k (X_m^k, \theta_k)}{\lambda_i + 1}.
\end{equation}
Note that if we fit linear models that generalize in the bottom level, we have $\mathrm{Var} (\sum_{e_{i, k} \in E} e_{i, k} ~g_k (X_m^k, \theta_k)) = \mathcal{O}(\frac{1}{m})$ (for other models, the variance should be at least in a smaller scale than $\mathcal{O}(1)$, which is the variance of observed samples). Therefore, by alternating $\lambda_i$:
\begin{itemize}
    \item $\mathrm{Var}(g_i (X_m^i, \theta_i)) = \mathrm{Var}(Y_m^i) = \mathcal{O}(1), \mathrm{when} ~\lambda_i \rightarrow 0$.
    \Item
    \begin{align}\hspace{-1.75em}
    	\mathrm{Var}(g_i(X_m^i, \theta_i)) & = \mathrm{Var}(\sum_{e_{i, k} \in E} e_{i, k} ~ g_k (X_m^k, \theta_k)) \nonumber \\
    	& = \mathcal{O}(\frac{1}{m}), \mathrm{when} ~ \lambda_i \rightarrow \infty.
    \end{align}
\end{itemize}
\end{pro}

This tells us that by alternating the coefficient $\lambda_i$, the amount of estimator variance at higher aggregation levels can be controlled. If lower-level models are accurate, then we can improve the higher-level models by this method. Instead of adding a hard coherency requirement like the post-inference methods, SHARQ provides more flexibility for controlling the variations.
\section{Experimental Results} \label{experimental_results}
In this section, we validate the performance of SHARQ on multiple hierarchical time series datasets with different properties and use cases. The experiments are conducted on both simulated and real-world data. Results demonstrate that SHARQ can generate coherent and accurate forecasts and well-capture the prediction uncertainty at any specified level.

\begin{table*}[t!]
\centering
\renewcommand\arraystretch{1.2}
\caption{Performance measured by MAPE \citep{makridakis2000m3} on Australian Labour (755 time series), and M5 competition (42840 time series), lower values are better. Level 1 is the top aggregation level, and 4 is the bottom level.}
\scalebox{0.7}{
\begin{tabular}{c|cccc|cccc|cccc|cccc}
\hlinewd{1.5pt}
Algorithm & \multicolumn{4}{c}{ RNN } & \multicolumn{4}{c}{ Autoregressive } & \multicolumn{4}{c}{ LST-Skip } & \multicolumn{4}{c}{ N-Beats } \\
\hline
\multirow{2}{*}{ Reconciliation } & \multicolumn{4}{c}{ Level } & \multicolumn{4}{c}{ Level } & \multicolumn{4}{c}{ Level } & \multicolumn{4}{c}{ Level } \\ \hhline{|~|----|----|----|----|}
& 1 & 2 & 3 & 4 & 1 & 2 & 3 & 4 & 1 & 2 & 3 & 4 & 1 & 2 & 3 & 4 \\ \hline
BU & 15.23 & 15.88 & 19.41 & \textbf{17.96} & 19.29 & 20.14 & 21.09 & 22.13 & 16.13 & 17.59 & 16.88 & \textbf{17.17} & 14.23 & 14.75 & 15.67 & \textbf{15.84} \\
Base & 12.89 & 14.26 & 16.96 & \textbf{17.96} & \textbf{17.59} & 19.86 & 20.98 & 22.13 & 14.99 & 12.31 & \textbf{15.12} & \textbf{17.17} & 12.18 & 13.32 & \textbf{14.32} & \textbf{15.84} \\
MinT-sam & 14.98 & 15.94 & 17.79 & 19.23 & 18.82 & 19.98 & 21.59 & 22.26 & 15.12 & 14.41 & 16.42 & 18.62 & 13.11 & 14.63 & 14.86 & 15.96 \\
MinT-shr & 14.46 & 15.43 & 16.94 & 18.75 & 18.54 & 19.98 & 21.22 & 22.01 & 15.06 & 13.89 & 16.31 & 17.56 & 12.76 & 14.41 & 14.77 & 15.87 \\
MinT-ols & 15.01 & 15.96 & 18.75 & 19.21 & 19.14 & 20.02 & 21.74 & 22.34 & 15.12 & 14.41 & 16.41 & 18.74 & 13.29 & 14.49 & 14.85 & 16.83 \\
ERM & 14.73 & 16.62 & 19.51 & 20.13 & 17.89 & 20.11 & 20.33 & \textbf{21.93} & 16.61 & 16.84 & 18.75 & 19.21 & 14.52 & 15.26 & 17.02 & 17.29 \\
SHARQ & \textbf{12.55} & \textbf{13.21} & \textbf{16.01} & \textbf{17.96} & 17.65 & \textbf{19.72} & \textbf{20.01} & 22.13 & \textbf{11.97} & \textbf{12.24} & 15.64 & \textbf{17.17} & \textbf{11.86} & \textbf{12.35} & 14.53 & \textbf{15.84} \\ \hhline{|=|====|====|====|====|}
BU & 11.42 & 12.04 & 12.32 & \textbf{11.72} & 12.77 & 14.59 & 16.11 & \textbf{16.56} & 10.11 & 12.69 & 10.78 & \textbf{10.94} & 11.01 & 11.05 & 12.43 & \textbf{11.34} \\
Base & 10.63 & 10.15 & 11.23 & \textbf{11.72} & 11.64 & 13.91 & 15.67 & \textbf{16.56} & 8.96 & 11.38 & 10.59 & \textbf{10.94} & 9.64 & 9.88 & 11.11 & \textbf{11.34} \\
MinT-sam & 11.25 & 11.67 & 11.87 & 12.99 & 12.34 & 14.09 & 15.97 & 17.54 & 9.64 & 12.31 & 11.02 & 11.01 & 9.97 & 10.82 & 11.89 & 12.77 \\
MinT-shr & 10.76 & 11.03 & 11.49 & 12.81 & 11.92 & 13.85 & 15.76 & 17.33 & 9.19 & 11.97 & 10.71 & 10.99 & 9.78 & 10.69 & 11.56 & 12.63 \\
MinT-ols & 11.75 & 11.56 & 12.06 & 13.05 & 12.32 & 14.21 & 15.97 & 17.56 & 9.63 & 12.54 & 10.98 & 11.02 & 10.41 & 11.01 & 12.02 & 12.71 \\
ERM & 11.86 & 12.01 & 12.42 & 13.54 & 12.61 & 14.02 & \textbf{15.41} & 17.14 & 10.35 & 13.01 & 13.15 & 13.56 & 10.44 & 11.22 & 13.42 & 13.96 \\
SHARQ & \textbf{9.87} & \textbf{9.68} & \textbf{10.41} & \textbf{11.72} & \textbf{11.23} & \textbf{13.84} & 15.69 & \textbf{16.56} & \textbf{8.68} & \textbf{9.49} & \textbf{10.23} & \textbf{10.94} & \textbf{9.67} & \textbf{9.76} & \textbf{10.75} & \textbf{11.34} \\
\hlinewd{1.5pt}
\end{tabular}\hspace{5em}}
\label{tab:mape}
\end{table*}

\subsection{Simulation Experiments}

We first demonstrate that quantile loss can handle various kinds of error distributions and thus provide more stable and accurate uncertainty estimation than methods under Gaussian error assumption. We trained vanilla RNN models on three different sequential data with distinct error distributions. We implemented a model that has multiple quantile estimators with shared features, which enables more efficient training. The bagging method \citep{oliveira2014ensembles} is used as a baseline where we utilize model ensembles to produce confidence intervals. Figure \ref{fig:sim_plot} shows the advantage of quantile estimators on simulated sequential data, and Table \ref{tab:sim_res} compares forecasting results as well as demonstrates our model structure. Although it is difficult to capture the trend of discontinuous functions in Figure \ref{fig:sim_plot} (b) and (c), the quantile estimators are accurate and stable under both skewed and symmetric error distributions, where it also outperforms the baseline for all types of error distributions.

\subsection{Hierarchical Time Series}
We validate the performance of SHARQ on multiple real-world hierarchical time-series datasets, which include Australian Labour, FTSE \citep{doherty2005hierarchical}, M5 competition, and Wikipedia webpage views dataset (see Appendix \ref{sec:data} for details). This type of data usually contains categorical features (e.g., locations, genders) that can be used to aggregate across time series to construct hierarchical structures. We compare our method with state-of-the-art reconciliation algorithms MinT \citep{wickramasuriya2019optimal} and ERM \citep{ben2019regularized}, along with other baselines, including bottom-up (BU) and base forecast. To have a fair comparison, we first pre-process each dataset using information from categorical features. The bottom-up training procedure in Algorithm \ref{alg:bu} is then used for each method except for BU. Specifically, the model training settings of the base forecast, MinT and ERM are by default the same as SHARQ, except that they do not have mean and quantile reconciliation. As for MinT and ERM, extra reconciliations are performed after model training. In this case, the algorithm has access to the hierarchical information about the dataset. We also incorporate different time series forecasting algorithms into our framework, which ranges from linear auto-regressive model and RNN-GRU \citep{chung2014empirical} to advanced models such as LSTNet \citep{lai2018modeling} and N-Beats \citep{oreshkin2019n}. Although these models are not originally designed for hierarchical time series problems, we show that the performance on this task can be improved under our framework. 

Table \ref{tab:mape} and \ref{tab:crps} shows forecasting results across all reconciliation methods and models on Australian Labour (upper) and M5 (lower) dataset, the results are averaged across 3 runs. Specifically, MAPE \citep{makridakis2000m3} measures accuracy for point forecast by Eq.(\ref{contstrained_loss}), and $Continuous ~Ranked ~Probability ~Score$ (CRPS) \citep{matheson1976scoring} measures the holistic accuracy of a probabilistic forecast, using multiple quantiles. Overall, SHARQ outperforms other reconciliation baselines, resulting in much lower MAPE and CRPS over all four models, particularly at the higher aggregation levels. Specifically, although the bottom-up training of SHARQ leads to the same bottom level performance as BU and Base method, the error accumulation and inconsistency across the hierarchy leads to higher error in other aggregation levels. More importantly, the better performance of SHARQ over Base and BU in multiple datasets validates the necessity of hierarchical construction in DNN training. Besides, comparing the autoregressive model results with others, SHARQ tends to perform better when the forecasting model is less parsimonious for the dataset. Figure \ref{fig:eight_fig} presents multi-step forecasting results, which possess the advantage of coherent estimation at multiple quantile levels. 

\begin{figure*}[t!]
    \begin{minipage}{\textwidth}
    \centering
    \begin{tabular}{@{\hspace{-2.4ex}} c @{\hspace{-2.4ex}} @{\hspace{-2.4ex}} c @{\hspace{-2.4ex}} @{\hspace{-2.4ex}} c @{\hspace{-2.4ex}} @{\hspace{-2.4ex}} c @{\hspace{-2.4ex}}}
        \begin{tabular}{c}
        \includegraphics[width=.276\textwidth]{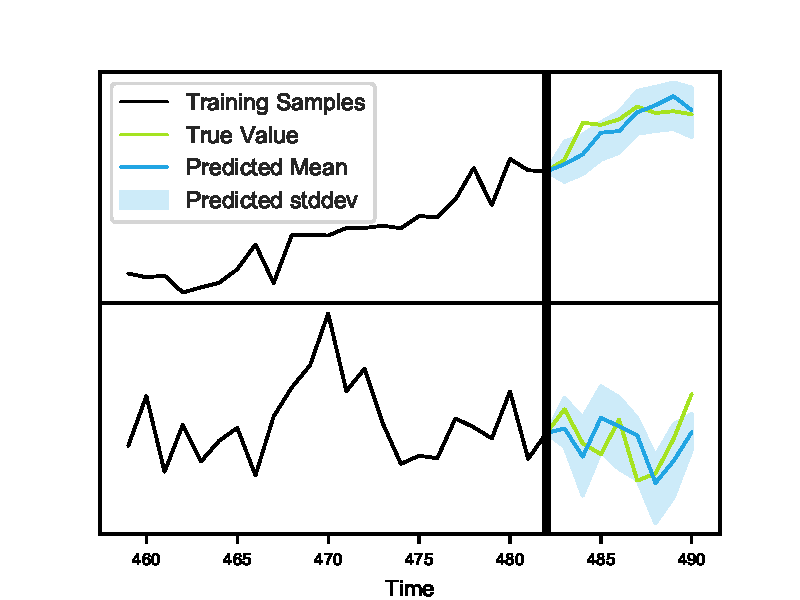}
        \\
        {\small{(a) Labour (SHARQ)}}
        \end{tabular} &
        \begin{tabular}{c}
        \includegraphics[width=.276\textwidth]{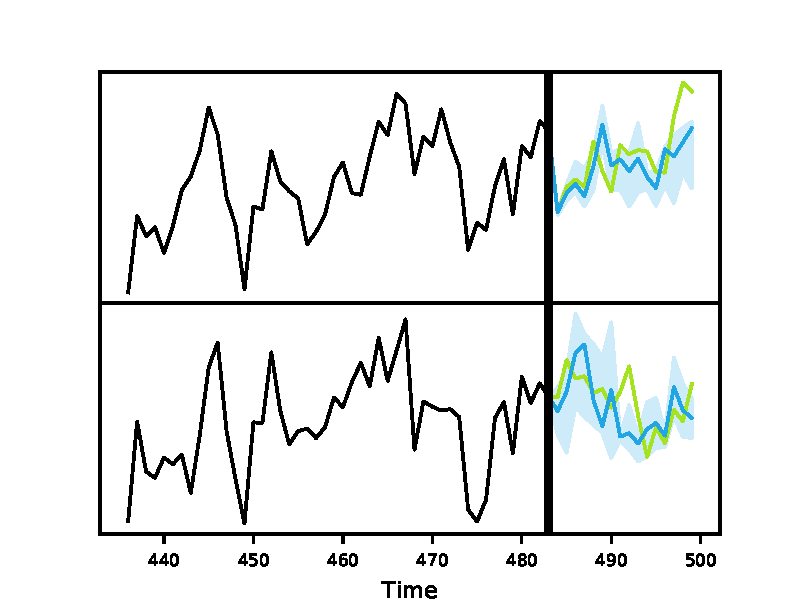}
        \\
        {\small{(b) FTSE (SHARQ)}}
        \end{tabular} & 
        \begin{tabular}{c}
        \includegraphics[width=.276\textwidth]{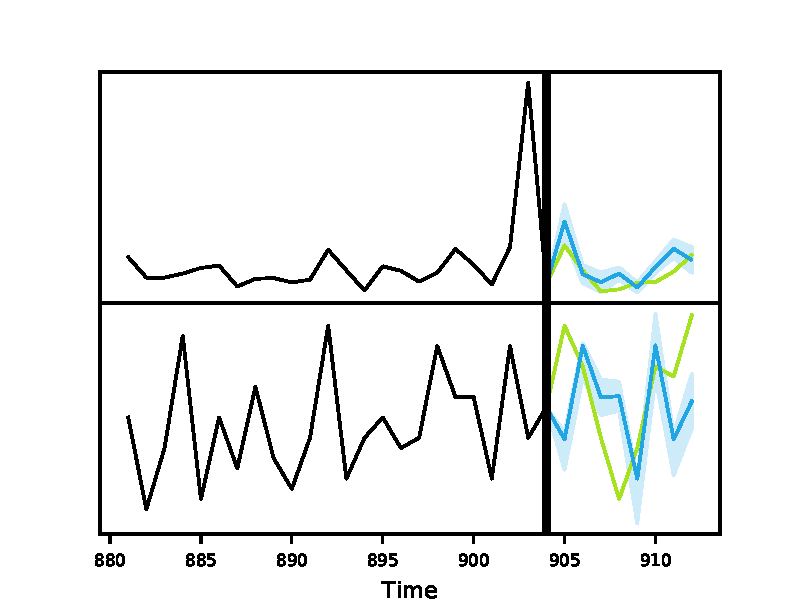} 
        \\
        {\small{(c) M5 (SHARQ)}}
        \end{tabular} &
        \begin{tabular}{c}
        \includegraphics[width=.276\textwidth]{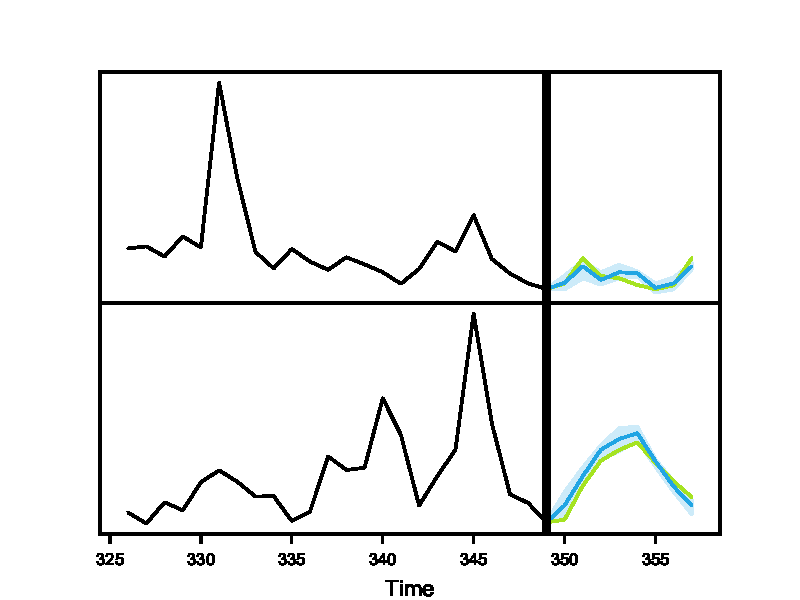}
        \\
        {\small{(d) Wiki (SHARQ)}}
        \end{tabular} \\
        \begin{tabular}{c}
        \includegraphics[width=.276\textwidth]{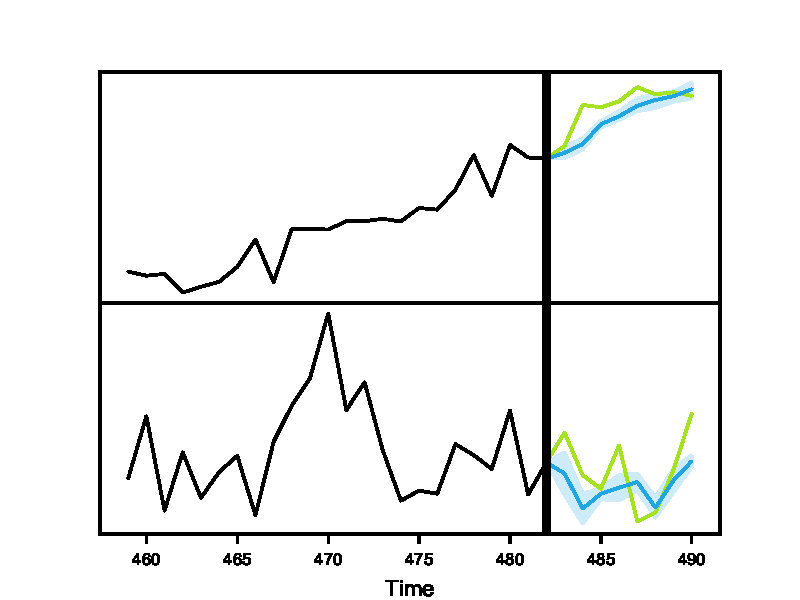}
        \\
        {\small{(a) Labour (MinT-shr)}}
        \end{tabular} &
        \begin{tabular}{c}
        \includegraphics[width=.276\textwidth]{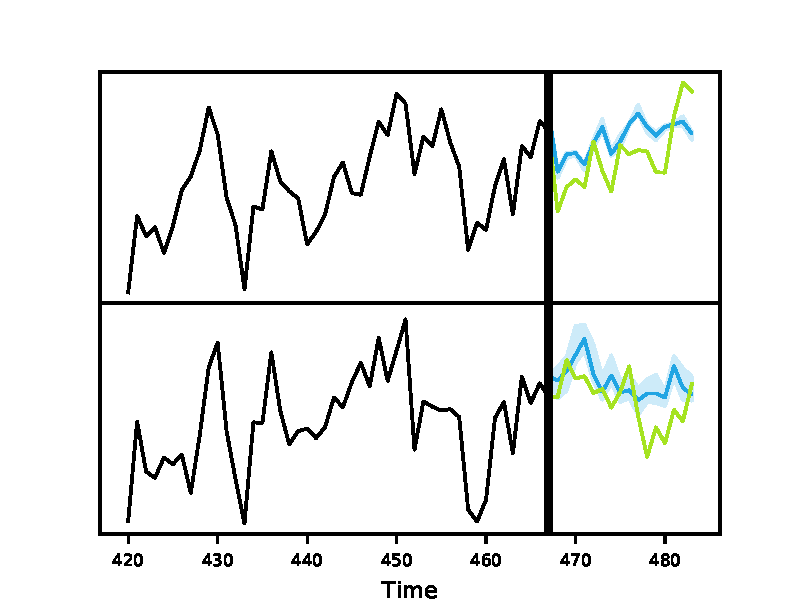}
        \\
        {\small{(b) FTSE (MinT-ols)}}
        \end{tabular} & 
        \begin{tabular}{c}
        \includegraphics[width=.276\textwidth]{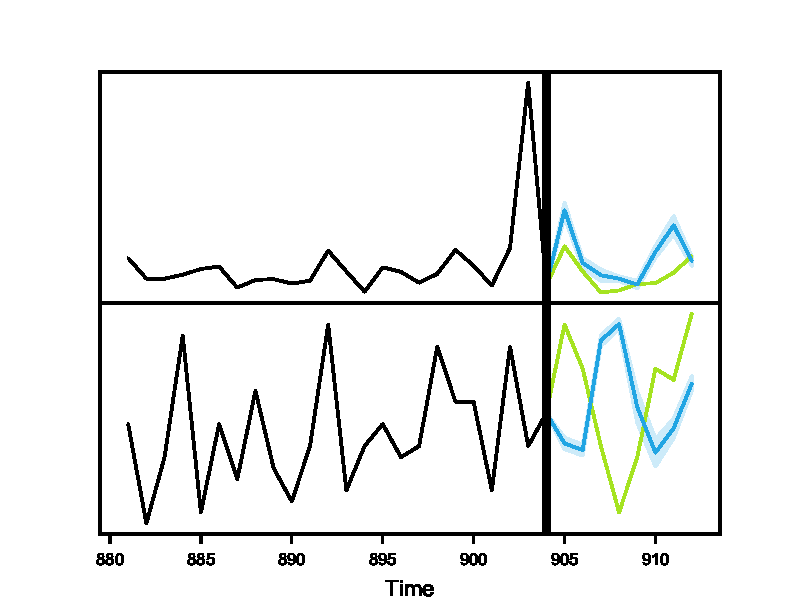} 
        \\
        {\small{(c) M5 (MinT-shr)}}
        \end{tabular} &
        \begin{tabular}{c}
        \includegraphics[width=.276\textwidth]{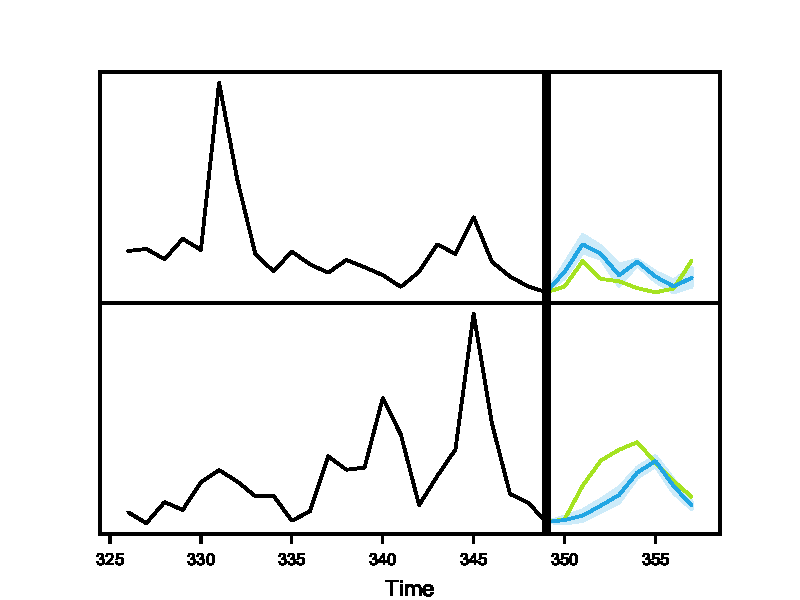}
        \\
        {\small{(d) Wiki (MinT-sam)}}
        \end{tabular} \\
        \end{tabular}
    \end{minipage}
    \caption{Top and bottom level forecasts on four datasets using the LSTNet skip connection model. For each dataset, we plot the results of SHARQ and the second-best reconciliation method. P5 and P95 forecasts are the lower and upper boundaries of the forecast band. We use mean as the point estimator (also complement of P50) for all bottom-level data and other aggregation levels of Australian Labour and FTSE data.}
    \label{fig:eight_fig}
\end{figure*}

\begin{table*}[t!]
\centering
\renewcommand\arraystretch{1.2}
\scalebox{0.7}{
\begin{tabular}{c|cccc|cccc|cccc|cccc}
\hlinewd{1.5pt}
Algorithm & \multicolumn{4}{c}{ RNN } & \multicolumn{4}{c}{ Autoregressive } & \multicolumn{4}{c}{ LST-Skip } & \multicolumn{4}{c}{ N-Beats } \\
\hline
\multirow{2}{*}{ Reconciliation } & \multicolumn{4}{c}{ Level } & \multicolumn{4}{c}{ Level } & \multicolumn{4}{c}{ Level } & \multicolumn{4}{c}{ Level } \\
\hhline{|~|----|----|----|----|}
& 1 & 2 & 3 & 4 & 1 & 2 & 3 & 4 & 1 & 2 & 3 & 4 & 1 & 2 & 3 & 4 \\ \hline
BU & 0.244 & 0.221 & 0.186 & \textbf{0.149} & 0.401 & 0.367 & 0.303 & 0.231 & 0.241 & 0.222 & 0.193 & \textbf{0.142} & 0.232 & 0.211 & 0.196 & 0.154 \\
Base & 0.119 & 0.135 & 0.143 & \textbf{0.149} & 0.174 & 0.203 & 0.221 & 0.231 & 0.124 & 0.139 & 0.142 & \textbf{0.142} & 0.122 & 0.141 & 0.141 & 0.154 \\
MinT-sam & 0.106 & 0.135 & 0.139 & 0.152 & 0.167 & 0.191 & 0.214 & 0.227 & 0.106 & 0.125 & 0.133 & 0.156 & 0.106 & 0.119 & 0.141 & \textbf{0.153} \\
MinT-shr & 0.103 & 0.129 & 0.137 & 0.158 & 0.162 & 0.189 & 0.206 & 0.232 & 0.101 & 0.113 & 0.132 & 0.153 & 0.103 & \textbf{0.114} & 0.137 & 0.155 \\
MinT-ols & 0.109 & 0.133 & 0.142 & 0.159 & 0.167 & 0.194 & 0.215 & 0.233 & 0.109 & 0.124 & 0.133 & 0.154 & 0.111 & 0.123 & 0.142 & 0.155 \\
ERM & 0.126 & 0.147 & 0.152 & 0.156 & 0.164 & \textbf{0.178} & \textbf{0.192} & \textbf{0.201} & 0.132 & 0.145 & 0.149 & 0.162 & 0.121 & 0.138 & 0.143 & 0.158 \\
SHARQ & \textbf{0.097} & \textbf{0.124} & \textbf{0.133} & \textbf{0.149} & \textbf{0.157} & 0.187 & 0.199 & 0.231 & \textbf{0.089} & \textbf{0.096} & \textbf{0.126} & \textbf{0.142} & \textbf{0.092} & 0.115 & \textbf{0.136} & 0.154 \\ \hhline{|=|====|====|====|====|}
BU & 0.247 & 0.231 & 0.226 & \textbf{0.208} & 0.397 & 0.375 & 0.316 & 0.297 & 0.219 & 0.211 & 0.194 & \textbf{0.164} & 0.199 & 0.171 & 0.152 & \textbf{0.135} \\
Base & 0.162 & 0.167 & 0.193 & \textbf{0.208} & 0.231 & 0.257 & 0.265 & 0.297 & 0.146 & 0.152 & 0.175 & \textbf{0.164} & 0.079 & 0.128 & 0.136 & \textbf{0.135} \\
MinT-sam & 0.147 & 0.158 & 0.154 & 0.211 & 0.257 & 0.262 & 0.271 & 0.279 & 0.112 & 0.141 & 0.175 & 0.189 & 0.091 & 0.124 & 0.142 & 0.149 \\
MinT-shr & 0.134 & 0.142 & 0.146 & 0.213 & 0.256 & 0.248 & 0.268 & 0.288 & 0.096 & 0.137 & 0.134 & 0.171 & 0.083 & 0.112 & 0.147 & 0.166 \\
MinT-ols & 0.143 & 0.161 & 0.154 & 0.215 & 0.259 & 0.261 & 0.272 & 0.283 & 0.109 & 0.154 & 0.156 & 0.191 & 0.086 & 0.117 & 0.139 & 0.162 \\
ERM & 0.152 & 0.154 & 0.188 & 0.226 & 0.213 & 0.229 & \textbf{0.241} & \textbf{0.267} & 0.124 & 0.166 & 0.168 & 0.194 & 0.098 & 0.129 & 0.151 & 0.172 \\
SHARQ & \textbf{0.071} & \textbf{0.063} & \textbf{0.114} & \textbf{0.208} & \textbf{0.189} & \textbf{0.225} & 0.279 & 0.297 & \textbf{0.069} & \textbf{0.074} & \textbf{0.108} & \textbf{0.164} & \textbf{0.067} & \textbf{0.069} & \textbf{0.096} & \textbf{0.135} \\
\hlinewd{1.5pt}
\end{tabular}\hspace{5em}}
\caption{Performance measured by CRPS \citep{matheson1976scoring} on Australian Labour (755 time series), and M5 competition (42840 time series), lower values are better. Level 1 is the top aggregation level, and 4 is the bottom aggregation level.}
\label{tab:crps}
\end{table*}

\begin{figure}[h]
    \centering{
    \includegraphics[width=.48\textwidth]{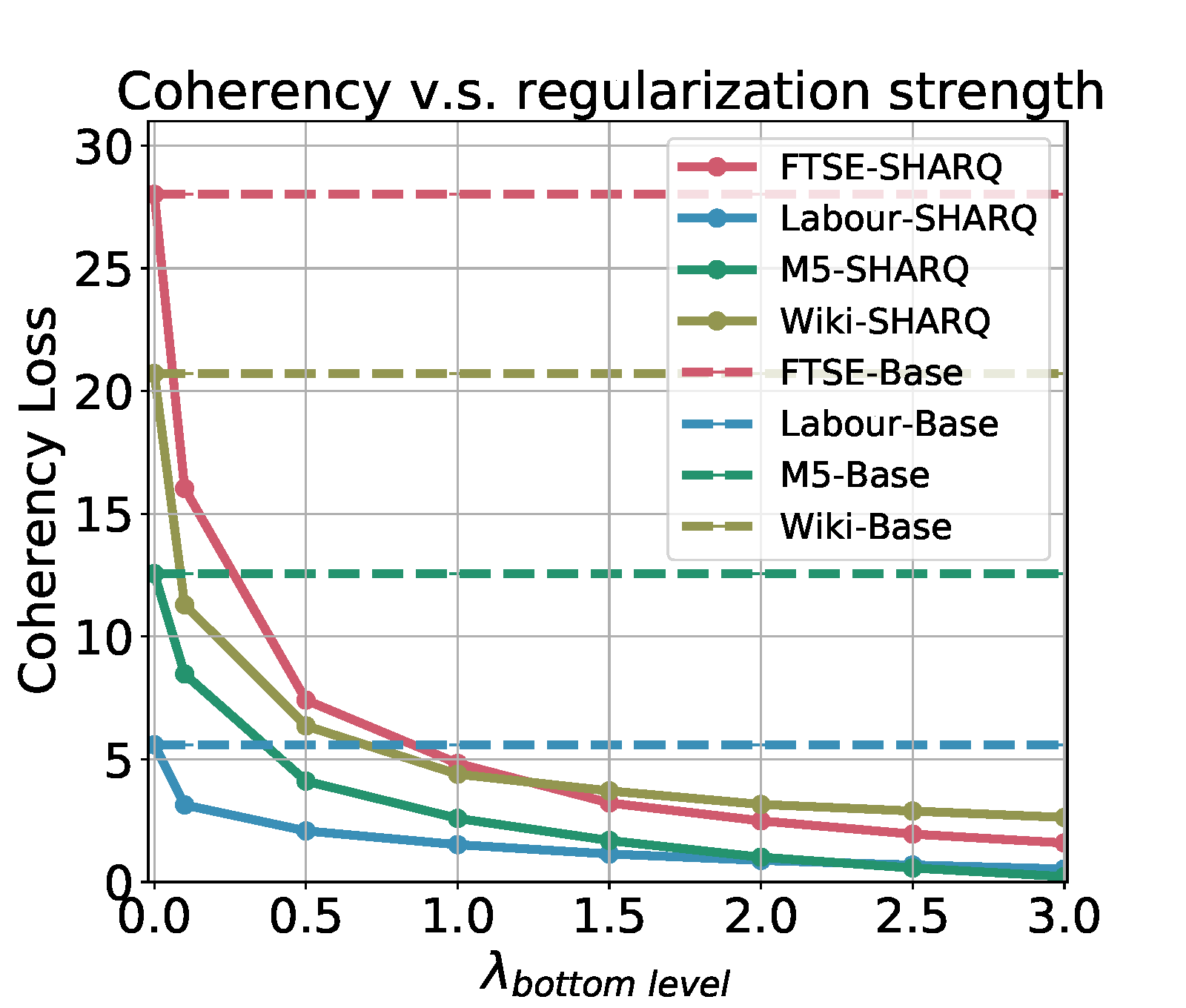}
    \caption{Coherency loss of SHARQ compared with the Base method on four datasets. Results are averaged across all the forecasting models.}}
    \label{fig:coherency}
\end{figure}

We pre-define the hyper-parameter $\lambda_i$ for vertex $v_i$ from Eq.(\ref{contstrained_loss}) level-wise, where we use the same $\lambda$s for all time series at the same level, and gradually decrease this value at higher aggregation levels. This is because time series at the same level possess similar magnitudes. With more vertices at lower levels, the chances of having consistency issues are higher, and error can accumulate to higher levels. 

We then evaluate the effect of regularization strength on forecasting coherency across the hierarchical structure; we summarize the result in Figure \ref{fig:coherency}, where the coherency loss drops dramatically after incorporating the hierarchical regularization at each level. Note that we mainly compare SHARQ with the Base method (SHARQ without regularization), as other reconciliation approaches generate absolute coherent results at the cost of sacrificing forecasting accuracy. More detailed evaluations can be found in Appendix \ref{sec:add_exp}.

\subsection{Comparison with Baseline Methods} 
Learning inter-level relationships through regularization helps SHARQ generalize better while mitigating coherency issues. It also provides a learnable trade-off between coherency and accuracy. From a computational perspective, MinT and ERM require one to compute matrix inversion explicitly. Note that the ERM method could compute the weight matrix on a validation set, but additional matrix computations are required during inference. Crucially, they depend on the Gaussian and unbiasedness assumptions, as stated in \citet{ben2019regularized, hyndman2011optimal, wickramasuriya2015forecasting} and their performance degrades noticeably when faced with actual data that do not match these assumptions well.

\section{Conclusion} \label{conclusion}
This paper has proposed a distributed optimization framework to generate probabilistic forecasts for a set of time series subject to hierarchical constraints. In our approach, the forecasting model is trained in a bottom-up fashion. At any stage, the model training involves simultaneously updating model parameters at two adjacent levels while maintaining the coherency constraints. This enables manageable information exchange at different levels of data aggregation. Our framework can incorporate any forecasting model and the non-parametric quantile loss function to generate accurate and coherent forecasts with pre-specified confidence levels. We have analytically demonstrated that by training the model with our modified objective function, the variance of time series data at higher aggregation levels can be reduced. We also compared our method empirically with the state-of-the-art hierarchical forecasting methods with cutting-edge base forecasters. The results show that our method produces relatively robust with accurate and coherent forecasts. Our proposed method reduces the inference complexity compared to the state-of-the-art algorithms, which perform a computationally expensive matrix inversion operation during the inference to achieve the reconciliation. 

As for future work, we plan to extend our method to multi-variate time series to be forecast at different time granularities while obeying hierarchical relationships. Besides, we also plan to investigate incorporating exogenous variables and related metadata. 

\textbf{Acknowledgements} \\
This work is supported by Intuit Inc. The authors would like to thank all reviewers for their constructive feedback and Tongzheng Ren for helpful discussion.
\bibliography{references_dis_opt}

\begin{thebibliography}{40}
\providecommand{\natexlab}[1]{#1}
\providecommand{\url}[1]{\texttt{#1}}
\expandafter\ifx\csname urlstyle\endcsname\relax
  \providecommand{\doi}[1]{doi: #1}\else
  \providecommand{\doi}{doi: \begingroup \urlstyle{rm}\Url}\fi

\bibitem[sas()]{sasforecasting}
Sas® forecasting for midsize business.

\bibitem[Athanasopoulos et~al.(2017)Athanasopoulos, Hyndman, Kourentzes, and
  Petropoulos]{athanasopoulos2017forecasting}
George Athanasopoulos, Rob~J Hyndman, Nikolaos Kourentzes, and Fotios
  Petropoulos.
\newblock Forecasting with temporal hierarchies.
\newblock \emph{European Journal of Operational Research}, 262\penalty0
  (1):\penalty0 60--74, 2017.

\bibitem[Ben~Taieb and Koo(2019)]{ben2019regularized}
Souhaib Ben~Taieb and Bonsoo Koo.
\newblock Regularized regression for hierarchical forecasting without
  unbiasedness conditions.
\newblock In \emph{Proceedings of the 25th ACM SIGKDD International Conference
  on Knowledge Discovery \& Data Mining}, pages 1337--1347, 2019.

\bibitem[Blundell et~al.(2015)Blundell, Cornebise, Kavukcuoglu, and
  Wierstra]{blundell2015weight}
Charles Blundell, Julien Cornebise, Koray Kavukcuoglu, and Daan Wierstra.
\newblock Weight uncertainty in neural networks.
\newblock \emph{arXiv preprint arXiv:1505.05424}, 2015.

\bibitem[Boyd and Vandenberghe(2004)]{boyd2004convex}
Stephen Boyd and Lieven Vandenberghe.
\newblock \emph{Convex optimization}.
\newblock Cambridge university press, 2004.

\bibitem[Chung et~al.(2014)Chung, Gulcehre, Cho, and
  Bengio]{chung2014empirical}
Junyoung Chung, Caglar Gulcehre, KyungHyun Cho, and Yoshua Bengio.
\newblock Empirical evaluation of gated recurrent neural networks on sequence
  modeling.
\newblock \emph{arXiv preprint arXiv:1412.3555}, 2014.

\bibitem[Doherty et~al.(2005)Doherty, Adams, Davey, and
  Pensuwon]{doherty2005hierarchical}
Kevin~AJ Doherty, Rod~G Adams, Neil Davey, and Wanida Pensuwon.
\newblock Hierarchical topological clustering learns stock market sectors.
\newblock In \emph{2005 ICSC Congress on Computational Intelligence Methods and
  Applications}, pages 6--pp. IEEE, 2005.

\bibitem[Franceschi et~al.(2019)Franceschi, Niepert, Pontil, and
  He]{franceschi2019learning}
Luca Franceschi, Mathias Niepert, Massimiliano Pontil, and Xiao He.
\newblock Learning discrete structures for graph neural networks.
\newblock \emph{arXiv preprint arXiv:1903.11960}, 2019.

\bibitem[Gasthaus et~al.(2019)Gasthaus, Benidis, Wang, Rangapuram, Salinas,
  Flunkert, and Januschowski]{gasthaus2019probabilistic}
Jan Gasthaus, Konstantinos Benidis, Yuyang Wang, Syama~Sundar Rangapuram, David
  Salinas, Valentin Flunkert, and Tim Januschowski.
\newblock Probabilistic forecasting with spline quantile function rnns.
\newblock In \emph{The 22nd International Conference on Artificial Intelligence
  and Statistics}, pages 1901--1910, 2019.

\bibitem[Hyndman et~al.(2011)Hyndman, Ahmed, Athanasopoulos, and
  Shang]{hyndman2011optimal}
Rob~J Hyndman, Roman~A Ahmed, George Athanasopoulos, and Han~Lin Shang.
\newblock Optimal combination forecasts for hierarchical time series.
\newblock \emph{Computational Statistics \& Data Analysis}, 55\penalty0
  (9):\penalty0 2579--2589, 2011.

\bibitem[Hyndman et~al.(2016)Hyndman, Lee, and Wang]{hyndman2016fast}
Rob~J Hyndman, Alan~J Lee, and Earo Wang.
\newblock Fast computation of reconciled forecasts for hierarchical and grouped
  time series.
\newblock \emph{Computational statistics \& data analysis}, 97:\penalty0
  16--32, 2016.

\bibitem[Iwata and Ghahramani(2017)]{iwata2017improving}
Tomoharu Iwata and Zoubin Ghahramani.
\newblock Improving output uncertainty estimation and generalization in deep
  learning via neural network gaussian processes.
\newblock \emph{arXiv preprint arXiv:1707.05922}, 2017.

\bibitem[Kuleshov et~al.(2018)Kuleshov, Fenner, and
  Ermon]{kuleshov2018accurate}
Volodymyr Kuleshov, Nathan Fenner, and Stefano Ermon.
\newblock Accurate uncertainties for deep learning using calibrated regression.
\newblock \emph{arXiv preprint arXiv:1807.00263}, 2018.

\bibitem[Lachapelle et~al.(2019)Lachapelle, Brouillard, Deleu, and
  Lacoste-Julien]{lachapelle2019gradient}
S{\'e}bastien Lachapelle, Philippe Brouillard, Tristan Deleu, and Simon
  Lacoste-Julien.
\newblock Gradient-based neural dag learning.
\newblock \emph{arXiv preprint arXiv:1906.02226}, 2019.

\bibitem[Lai et~al.(2018)Lai, Chang, Yang, and Liu]{lai2018modeling}
Guokun Lai, Wei-Cheng Chang, Yiming Yang, and Hanxiao Liu.
\newblock Modeling long-and short-term temporal patterns with deep neural
  networks.
\newblock In \emph{The 41st International ACM SIGIR Conference on Research \&
  Development in Information Retrieval}, pages 95--104, 2018.

\bibitem[Lakshminarayanan et~al.(2017)Lakshminarayanan, Pritzel, and
  Blundell]{lakshminarayanan2017simple}
Balaji Lakshminarayanan, Alexander Pritzel, and Charles Blundell.
\newblock Simple and scalable predictive uncertainty estimation using deep
  ensembles.
\newblock In \emph{Advances in neural information processing systems}, pages
  6402--6413, 2017.

\bibitem[Lauderdale et~al.(2019)Lauderdale, Bailey, Blumenau, and
  Rivers]{lauderdale2019model}
Benjamin~E Lauderdale, Delia Bailey, Jack Blumenau, and Douglas Rivers.
\newblock Model-based pre-election polling for national and sub-national
  outcomes in the us and uk.
\newblock \emph{International Journal of Forecasting}, 2019.

\bibitem[Li and Zhu(2008)]{li20081}
Youjuan Li and Ji~Zhu.
\newblock L 1-norm quantile regression.
\newblock \emph{Journal of Computational and Graphical Statistics}, 17\penalty0
  (1):\penalty0 163--185, 2008.

\bibitem[Li and Hoiem(2020)]{li2020improving}
Zhizhong Li and Derek Hoiem.
\newblock Improving confidence estimates for unfamiliar examples.
\newblock In \emph{Proceedings of the IEEE/CVF Conference on Computer Vision
  and Pattern Recognition}, pages 2686--2695, 2020.

\bibitem[Liu and Wu(2009)]{liu2009stepwise}
Yufeng Liu and Yichao Wu.
\newblock Stepwise multiple quantile regression estimation using non-crossing
  constraints.
\newblock \emph{Statistics and its Interface}, 2\penalty0 (3):\penalty0
  299--310, 2009.

\bibitem[Makridakis and Hibon(2000)]{makridakis2000m3}
Spyros Makridakis and Michele Hibon.
\newblock The m3-competition: results, conclusions and implications.
\newblock \emph{International journal of forecasting}, 16\penalty0
  (4):\penalty0 451--476, 2000.

\bibitem[Matheson and Winkler(1976)]{matheson1976scoring}
James~E Matheson and Robert~L Winkler.
\newblock Scoring rules for continuous probability distributions.
\newblock \emph{Management science}, 22\penalty0 (10):\penalty0 1087--1096,
  1976.

\bibitem[Mukherjee et~al.(2018)Mukherjee, Shankar, Ghosh, Tathawadekar,
  Kompalli, Sarawagi, and Chaudhury]{mukherjee2018armdn}
Srayanta Mukherjee, Devashish Shankar, Atin Ghosh, Nilam Tathawadekar, Pramod
  Kompalli, Sunita Sarawagi, and Krishnendu Chaudhury.
\newblock Armdn: Associative and recurrent mixture density networks for eretail
  demand forecasting.
\newblock \emph{arXiv preprint arXiv:1803.03800}, 2018.

\bibitem[Oliveira and Torgo(2014)]{oliveira2014ensembles}
Mariana~Rafaela Oliveira and Luis Torgo.
\newblock Ensembles for time series forecasting.
\newblock 2014.

\bibitem[Oreshkin et~al.(2019)Oreshkin, Carpov, Chapados, and
  Bengio]{oreshkin2019n}
Boris~N Oreshkin, Dmitri Carpov, Nicolas Chapados, and Yoshua Bengio.
\newblock N-beats: Neural basis expansion analysis for interpretable time
  series forecasting.
\newblock \emph{arXiv preprint arXiv:1905.10437}, 2019.

\bibitem[Russell(2017)]{russell2017ftse}
FTSE Russell.
\newblock Ftse uk index series.
\newblock \emph{Retrieved February}, 5:\penalty0 2017, 2017.

\bibitem[Salinas et~al.(2019)Salinas, Flunkert, Gasthaus, and
  Januschowski]{salinas2019deepar}
David Salinas, Valentin Flunkert, Jan Gasthaus, and Tim Januschowski.
\newblock Deepar: Probabilistic forecasting with autoregressive recurrent
  networks.
\newblock \emph{International Journal of Forecasting}, 2019.

\bibitem[Sen et~al.(2019)Sen, Yu, and Dhillon]{sen2019think}
Rajat Sen, Hsiang-Fu Yu, and Inderjit~S Dhillon.
\newblock Think globally, act locally: A deep neural network approach to
  high-dimensional time series forecasting.
\newblock In \emph{Advances in Neural Information Processing Systems}, pages
  4838--4847, 2019.

\bibitem[Sun et~al.(2019)Sun, Zhang, Shi, and Grosse]{sun2019functional}
Shengyang Sun, Guodong Zhang, Jiaxin Shi, and Roger Grosse.
\newblock Functional variational bayesian neural networks.
\newblock \emph{arXiv preprint arXiv:1903.05779}, 2019.

\bibitem[Taieb et~al.(2017)Taieb, Taylor, and Hyndman]{taieb2017coherent}
Souhaib~Ben Taieb, James~W Taylor, and Rob~J Hyndman.
\newblock Coherent probabilistic forecasts for hierarchical time series.
\newblock In \emph{Proceedings of the 34th International Conference on Machine
  Learning-Volume 70}, pages 3348--3357. JMLR. org, 2017.

\bibitem[Wen et~al.(2017)Wen, Torkkola, Narayanaswamy, and
  Madeka]{wen2017multi}
Ruofeng Wen, Kari Torkkola, Balakrishnan Narayanaswamy, and Dhruv Madeka.
\newblock A multi-horizon quantile recurrent forecaster.
\newblock \emph{arXiv preprint arXiv:1711.11053}, 2017.

\bibitem[Wickramasuriya et~al.(2015)Wickramasuriya, Athanasopoulos, Hyndman,
  et~al.]{wickramasuriya2015forecasting}
Shanika~L Wickramasuriya, George Athanasopoulos, Rob~J Hyndman, et~al.
\newblock Forecasting hierarchical and grouped time series through trace
  minimization.
\newblock \emph{Department of Econometrics and Business Statistics, Monash
  University}, 2015.

\bibitem[Wickramasuriya et~al.(2019)Wickramasuriya, Athanasopoulos, and
  Hyndman]{wickramasuriya2019optimal}
Shanika~L Wickramasuriya, George Athanasopoulos, and Rob~J Hyndman.
\newblock Optimal forecast reconciliation for hierarchical and grouped time
  series through trace minimization.
\newblock \emph{Journal of the American Statistical Association}, 114\penalty0
  (526):\penalty0 804--819, 2019.

\bibitem[Wu et~al.(2020)Wu, Pan, Long, Jiang, Chang, and
  Zhang]{wu2020connecting}
Zonghan Wu, Shirui Pan, Guodong Long, Jing Jiang, Xiaojun Chang, and Chengqi
  Zhang.
\newblock Connecting the dots: Multivariate time series forecasting with graph
  neural networks.
\newblock \emph{arXiv preprint arXiv:2005.11650}, 2020.

\bibitem[Yu et~al.(2017)Yu, Yin, and Zhu]{yu2017spatio}
Bing Yu, Haoteng Yin, and Zhanxing Zhu.
\newblock Spatio-temporal graph convolutional networks: A deep learning
  framework for traffic forecasting.
\newblock \emph{arXiv preprint arXiv:1709.04875}, 2017.

\bibitem[Yu et~al.(2019)Yu, Chen, Gao, and Yu]{yu2019dag}
Yue Yu, Jie Chen, Tian Gao, and Mo~Yu.
\newblock Dag-gnn: Dag structure learning with graph neural networks.
\newblock \emph{arXiv preprint arXiv:1904.10098}, 2019.

\bibitem[Zhang et~al.(2020)Zhang, Cui, and Zhu]{zhang2020deep}
Ziwei Zhang, Peng Cui, and Wenwu Zhu.
\newblock Deep learning on graphs: A survey.
\newblock \emph{IEEE Transactions on Knowledge and Data Engineering}, 2020.

\bibitem[Zhao et~al.(2016)Zhao, Chen, Lu, and Ramakrishnan]{zhao2016multi}
Liang Zhao, Feng Chen, Chang-Tien Lu, and Naren Ramakrishnan.
\newblock Multi-resolution spatial event forecasting in social media.
\newblock In \emph{2016 IEEE 16th International Conference on Data Mining
  (ICDM)}, pages 689--698. IEEE, 2016.

\bibitem[Zheng et~al.(2018)Zheng, Aragam, Ravikumar, and Xing]{zheng2018dags}
Xun Zheng, Bryon Aragam, Pradeep~K Ravikumar, and Eric~P Xing.
\newblock Dags with no tears: Continuous optimization for structure learning.
\newblock In \emph{Advances in Neural Information Processing Systems}, pages
  9472--9483, 2018.

\bibitem[Zhu and Laptev(2017)]{zhu2017deep}
Lingxue Zhu and Nikolay Laptev.
\newblock Deep and confident prediction for time series at uber.
\newblock In \emph{2017 IEEE International Conference on Data Mining Workshops
  (ICDMW)}, pages 103--110. IEEE, 2017.

\end{thebibliography}
\renewcommand{\L}{\mathcal{L}}
\renewcommand{\hy}{\hat{y}}
\renewcommand{\hb}{\hat{\beta}}
\renewcommand{\ty}{\Tilde{y}}
\renewcommand{\hY}{\hat{Y}}
\renewcommand{\tY}{\Tilde{Y}}
\renewcommand{\L}{\mathcal{L}}
\renewcommand{\TL}{\widetilde{\mathcal{L}}_n}
\renewcommand{\S}{\mathcal{S}}
\renewcommand{\I}{\mathcal{I}}


%





%


\newpage

\appendix
\begin{center}
\huge \textbf{Appendix}
\end{center}

\section{Further Discussion on Related Works} \label{sec:background}
As we mentioned in Section (\ref{sec:intro}), state-of-the-art hierarchical forecasting algorithms \citep{ben2019regularized, hyndman2011optimal, hyndman2016fast, wickramasuriya2015forecasting} involves computing the optimal $P$ matrix to combine the base forecasts under different situations linearly. We now summarize these methods as follows.

\subsection{Generalized Least Squares (GLS) Reconciliation} 
Denote $b_t \in \mathbb{R}^m$, $a_t \in \mathbb{R}^k$ as the observations at time $t$ for the $m$ and $k$ series at the bottom and aggregation level(s), respectively. $S \in \{0, 1\}^{n \times m}$ is the summing matrix. Each entry $S_{ij}$ equals to 1 if the $i^{th}$ aggregate series contains the $j^{th}$ bottom-level series, where $i = 1, ..., k$ and $j = 1, ..., m$. Denote $\I_T = \{y_1, y_2, ..., y_T\}$ as the time series data observed up to time $T$; $\hat{b}_T(h)$ and $\hat{y}_T(h)$ as the $h-$step ahead forecast on the bottom-level and all levels based on $\I_T$.

Let $\hat{e}_T(h) = y_{T+h} - \hy_T(h)$ be the $h-$step ahead conditional base forecast errors and $\beta_T(h) = E[\hat{b}_T(h)~|~\I_T]$ be the bottom-level mean forecasts. We then have $E[\hy_T(h)~|~\I_T] = S\beta_T (h)$. Assume that $E[\hat{e}_T(h)~|~\I_T] = 0$, then a set of reconciled forecasts will be unbiased iff $SPS = S$, i.e., \textbf{Assumption A1:}
\begin{equation}
\E [\ty_T (h) | \I_T] = \E [\hy_T (h) | \I_T] = S \beta_T(h) \quad
\end{equation}
The optimal combination approach proposed by \citet{hyndman2011optimal}, is based on solving the above regression problem using the generalized least square method:
\begin{equation}
    \hy_T(h) = S\beta_T(h) + \varepsilon_h,
\end{equation}
where $\varepsilon_h$ is the independent coherency error with zero mean and $\mathrm{Var}(\varepsilon_h) = \Sigma_h$. The GLS estimator of $\beta_T(h)$ is given by
\begin{equation}
    \hat{\beta}_T(h) = (S'\Sigma'_h S)^{-1} S' \Sigma_h' \hat{y}_T(h),
    \label{equ:gls}
\end{equation}
which is an unbiased, minimum variance  estimator. The optimal $P$ is $(S'\Sigma'_h S)^{-1} S' \Sigma_h'$. The reconciled mean and variance can therefore be obtained accordingly.

\subsection{Trace Minimization (MinT) Reconciliation}
Defining the reconciliation error as $\Tilde{e}_T(h) = y_{T+h} - \ty_T(h)$, the original problem can also be formulated as 
\begin{align}
    & \underset{P \in \mathcal{P}}{\min} ~ \E[\|\Tilde{e}_T(h)\|_2^2~|~\I_T] \notag \\
    & \mathrm{subject ~ to} ~ \E[\ty_T(h) ~|~ \I_T] = \E[\ty_T(h) ~|~ \I_T]
    \label{equ:formulation}
\end{align}
If the assumption \textbf{A1} still holds, then minimizing Eq.(\ref{equ:formulation}) reduces to 
\begin{equation}
    \underset{P \in \mathcal{P}}{\min} ~ \mathrm{Tr}(\mathrm{Var}[\Tilde{e}_T(h)~|~\I_T]) \quad \mathrm{subject ~ to ~ \textbf{A1}},
\end{equation}
where Tr(.) denotes the trace of a matrix. In \citet{wickramasuriya2019optimal}, the proposed optimal solution of $P$ obtained by solving this problem is given by 
\begin{equation}
    P = (S'W_h^{-1}S)^{-1}S'W_h^{-1},
    \label{mint}
\end{equation}
where $W_h = \E[\hat{e}_T(h)\hat{e}_T'(h) ~|~ \I_T]$ is the variance-covariance matrix of the $h-$step-ahead base forecast errors, which is different from the coherence errors $\Sigma_h$ in GLS reconciliation method given in Eq.(\ref{equ:gls}). There are various covariance estimators for $W_h$ considered in \citet{wickramasuriya2019optimal}, the most effective one is the shrinkage estimator with diagonal target, and can be computed by
\begin{equation}
    \hat{W}_h = (1 - \alpha) \hat{W}_s + \alpha \hat{W}_d, \quad \hat{W}_s = \frac{1}{T}\sum_{t=1}^T \hat{e}_t(1)\hat{e}_t(1)',
    \label{mint_est}
\end{equation}
where $\hat{W}_d = \mathrm{diag}(\hat{W}_s)$ and $\alpha \in (0, 1]$.

\subsection{Empirical Risk Minimization (ERM) Reconciliation}
Most recently, \citet{ben2019regularized} proposed a new method to relax the unbiasedness assumption \textbf{A1}. Specifically, the objective function in (\ref{equ:formulation}) can be decomposed as 
\begin{align}
    & \E[\|y_{T+h} - \ty_T(h)\|_2^2~|~\I_T] \label{equ:mse} \\
    & = \|SP(\E[\hy_T(h) | \I_T] - \E[y_{T+h}|\I_T]) \nonumber \\ 
    & + (S - SPS)\E[b_{T+h}|\I_T]\|_2^2 \label{equ:bias}\\
    & + \mathrm{Tr} (\mathrm{Var}[y_{T+h} - \ty_T(h)|\I_T] \label{equ:var},
\end{align}
where (\ref{equ:bias}) and (\ref{equ:var}) are the bias and variance terms of the revised forecasts $\ty_T(h)$. The assumption \textbf{A1} in MinT method renders (\ref{equ:bias}) to 0. Obviously, directly minimize the objective in (\ref{equ:mse}) provides a more general form of reconciliation represented by  following empirical risk minimization (ERM) problem:
\begin{equation}
    \underset{P \in \mathcal{P}}{\min} \frac{1}{(T - T_1 - h + 1)n}\sum_{t=T_1}^{T-h} \|y_{t+h} - SP \hy_t(h)\|_2^2,
    \label{equ:erm}
\end{equation}
where $T_1 < T$ is the number of observations used for model fitting. Empirically, this method demonstrates better performance than MinT according to \citet{ben2019regularized}, particularly when the forecasting models are mis-specified.

\section{Non-Additive Property of Quantile Loss} \label{sec:nonadd}
Here we prove the non-additive property of quantile loss as mentioned in Section (2.2).
\begin{thm}\textbf{(Non-additive Property)}
Assume two independent random variables $X_1 \sim N(\mu_1, \sigma_1^2)$ and $X_2 \sim N(\mu_2, \sigma_2^2)$, and define $Y = X_1 + X_2$. Then $Q_{Y}(\tau) \neq Q_{X_1}(\tau) + Q_{X_2}(\tau)$.
\end{thm}
\begin{proof}
The $\tau^{th}$ quantile of $X_1$ is given by: 
\begin{equation}
Q_{X_1}(\tau) = F_{X_1}^{-1}(\tau) = inf \{x: F_{X_1}(x) \geq \tau\},
\end{equation}

where $F_{X_1} (x)$ is $\frac{1}{2} \left[1 + erf\left(\frac{x - \mu_1}{\sigma_1 \sqrt{2}}\right)\right]$, and $erf(x) = \frac{1}{\sqrt{\pi}} \int_{-x}^{x} e^{-t^2} dt$. Therefore, we can further get:

\begin{align*}
    Q_{X_1}(\tau) & = \mu_1 + \sigma_1 \Phi^{-1} (\tau) = \mu_1 + \sigma_1 \sqrt{2} ~ erf^{-1} (2\tau - 1) \\
    Q_{X_2}(\tau) & = \mu_2 + \sigma_2 \Phi^{-1} (\tau) = \mu_2 + \sigma_2 \sqrt{2} ~ erf^{-1} (2\tau - 1) \\
\end{align*}
According to the additive property of Gaussian distribution, we have $Y \sim N(\mu_1 + \mu_2, \sigma_1^2 + \sigma_2^2)$, and
\begin{align}
     Q_{Y}(\tau) & = \mu_1 + \mu_2 + \sqrt{\sigma_1^2 + \sigma_2^2} ~ \Phi^{-1} (\tau) \nonumber \\
     & = \mu_1 + \mu_2 + \sqrt{\sigma_1^2 + \sigma_2^2} ~ \sqrt{2} ~ erf^{-1} (2\tau - 1).
     \label{quantile}
\end{align}

Therefore, even if we have \texttt{i.i.d.}normal distribution with $Y = X_1 + X_2$, it still doesn't imply $Q_{Y}(\tau) = Q_{X_1}(\tau) + Q_{X_2}(\tau)$. The only case that the addition property hold true in any quantile is when $X_1 = C \times X_2$, where $C$ is arbitrary constant. Obviously, this is not applicable. 
\end{proof}
In fact, under Gaussian assumption, we have the following additive property holds for any $\tau$:

\begin{equation}
    (Q_{Y}^{\tau} - \mu_Y )^2 = (Q_{X_1}^{\tau} - \mu_{X_1})^2 + (Q_{X_2}^{\tau} - \mu_{X_2})^2.
    \label{add_quant}
\end{equation}

Since by Eq.(\ref{quantile}), the left hand side of Eq.(\ref{add_quant}) is $2 (\sigma_1^2 + \sigma_2^2) ~ \left[erf^{-1} (2\tau - 1)\right]^2,$ and the right hand side of Eq.(\ref{add_quant}) is $2\sigma_1^2 \left[erf^{-1} (2\tau - 1)\right]^2 + 2\sigma_2^2 \left[erf^{-1} (2\tau - 1)\right]^2.$ Therefore, the additive property holds for any $\tau$ assume the RVs follow Gaussian distribution.

\begin{table*}[t]
\centering
\caption{\centering{Details of four hierarchical time-series datasets. Note that hierarchical levels mean the number of aggregation levels from bottom to top in the hierarchical structure used in the experiments.}}
\vspace{.5em}
\scalebox{1.05}{
\begin{tabular}{c|c|c|c}
\hlinewd{1.5pt}
Dataset & Total number of time series & Total length of time series & Hierarchical Levels \\ \hline
FTSE & 73 & 2512 & 4 \\
M5 & 42840 & 1969 & 4 \\
Wiki & 145000 & 550 & 5 \\
Labour & 755 & 500 & 4 \\
\hlinewd{1.5pt}
\end{tabular}}
\label{data_table}
\end{table*}

\begin{figure*}[t]
\centering
{
\setlength{\tabcolsep}{1pt} 
\renewcommand{\arraystretch}{1} 
\begin{tabu}{cccc}
\hspace{-2.05em}
\includegraphics[width=.25\textwidth]{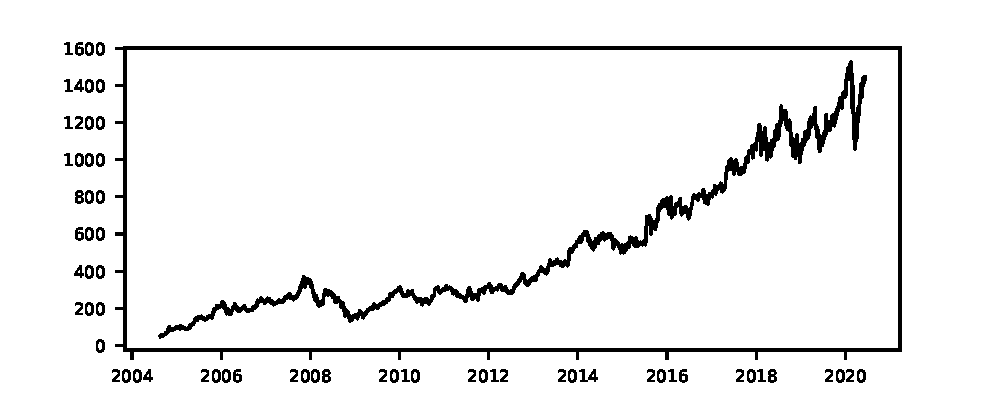}&
\hspace{-1.55em}
\includegraphics[width=.25\textwidth]{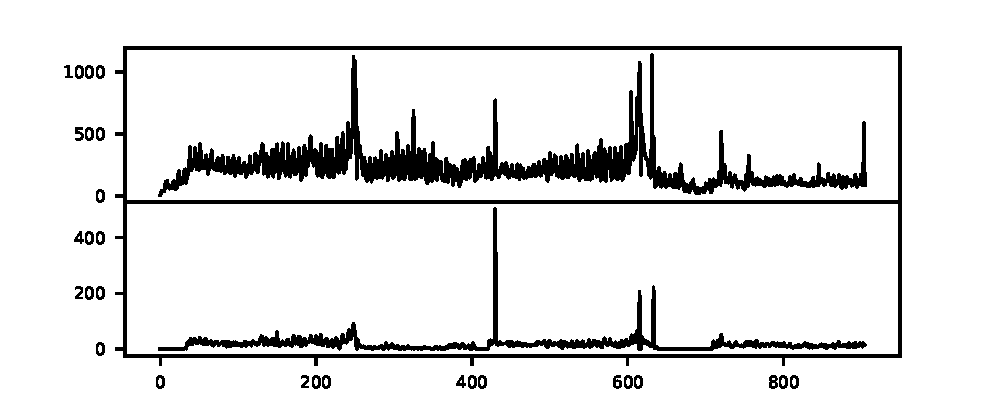}&
\hspace{-1.55em}
\includegraphics[width=.25\textwidth]{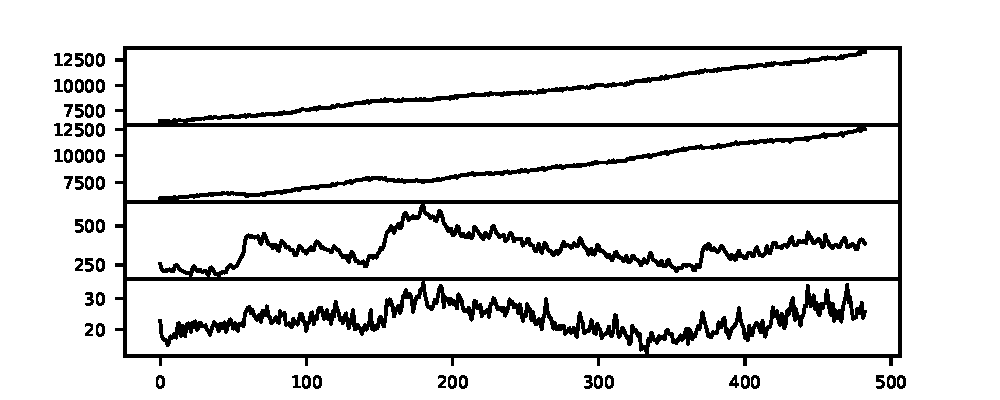}&
\hspace{-1.55em}
\includegraphics[width=.25\textwidth]{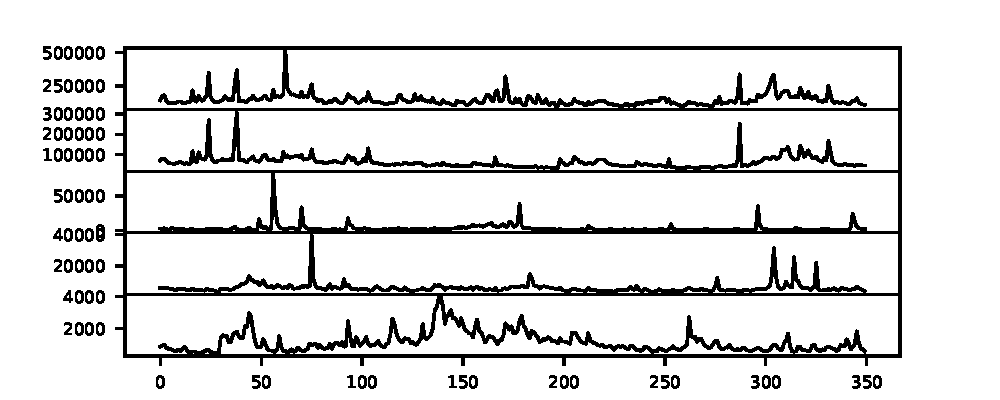}\\
\hspace{-2.45em} (a) FTSE & \hspace{-2.05em} (b) M5 & \hspace{-2.05em} (c) Labour & \hspace{-2.05em} (d) Wikipedia\\
\end{tabu}}
\caption{Visualization of hierarchical time series data. (a) Bottom level time series of FTSE (the open stock price of Google); (b) bottom and top level of unit sales record; (c) Australian Labour Force data at all aggregation levels; (d) Wikipedia page views data at all aggregation levels.}
\label{fig:raw_ts_plot}
\end{figure*}

\section{KKT Conditions} \label{kkt}
An alternative way of solving the optimization problem defined in Section (\ref{problem_formulation}) Eq.(\ref{contstrained_loss}) is to obtain the KKT conditions \citep{boyd2004convex}. For notational simplicity, we express the constrained loss for $i^{th}$ vertex and $m^{th}$ data point as $L_c (i, m)$. As the optimization problem is unconstrained, the KKT conditions will lead to:
\[ \frac{\partial }{ \partial [\theta_i, \Theta_i]} L_c (i, m) = [ ~ \frac{\partial }{ \partial \theta_i} L_c (i, m) ~,~ \frac{\partial }{ \partial \Theta_i} L_c (i, m) ~ ] = 0 , \]
which will further imply that
\begin{align*} 
& \lambda_i   \left [\frac{\partial }{\partial g_i}  L (g_i (X_m^i, \theta_i), ~ Y_m^i)  \right ]^{\top} \nonumber \\ 
& \left [ g_i (X_m^i, \theta_i) - \sum_{e_{i, k} \in E}  e_{i, k} ~ . ~g_k (X_m^k, \theta_k) \right ]\nonumber \\
& + \frac{\partial }{\partial g_i}  L (g_i (X_m^i, \theta_i), ~ Y_m^i)~ . ~ \frac{\partial g_i}{\partial \theta_i} = 0,
\end{align*}
and
\begin{align*}  
&\left ( e_{i, j} ~ . ~g_{j} (X_m^{j}, \theta_{j})  \right )^ T \nonumber \\
& \left ( g_i (X_m^i, \theta_i) - \sum_{e_{i, k} \in E}  e_{i, k} g_k (X_m^k, \theta_k) \right ) = 0, ~ \forall j | e_{i,j} \in E.
\end{align*}
However, we found that SHARQ performs better and more efficiently than the KKT approach during our empirical evaluation. Solving the KKT conditions requires matrix inversion in most situations. Besides, SHARQ is more flexible in incorporating various forecasting models and performs probabilistic forecasts.

\section{Dataset Details} \label{sec:data}
We first describe the details (dataset generation, processing, etc.) of each dataset used in the experiment. A summary of each dataset is shown in Table \ref{data_table}. Visualizations for some raw time series can be found in Figure \ref{fig:raw_ts_plot}.

\subsection{FTSE Stock Market Data}
The FTSE Global Classification System is a universally accepted classification scheme based on a market's division into Economic Groups, Industrial Sectors, and Industrial Sub-sectors. This system has been used to classify company data for over 30,000 companies from 59 countries. The FTSE 100 \citep{doherty2005hierarchical} is the top 100 capitalized blue-chip companies in the UK and is recognized as the measure of UK stock market performance \citep{russell2017ftse}. Base on the FTSE classification system, we formulate a 4-level hierarchical structure (Economic Groups, Industrial Sectors, Industrial Sub-sectors, and companies) of 73 companies in \citet{doherty2005hierarchical}. Our task is to model the stock market time series for each company. The stock market data of each company is available from the Yahoo Finance package\footnote[1]{https://pypi.org/project/yfinance/}. Since the stock market time series starting time of each company is not the same, we use a common time window ranging from January 4, 2010, to May 1, 2020.

\begin{figure*}[t]
    \centering
    \includegraphics[width=\textwidth]{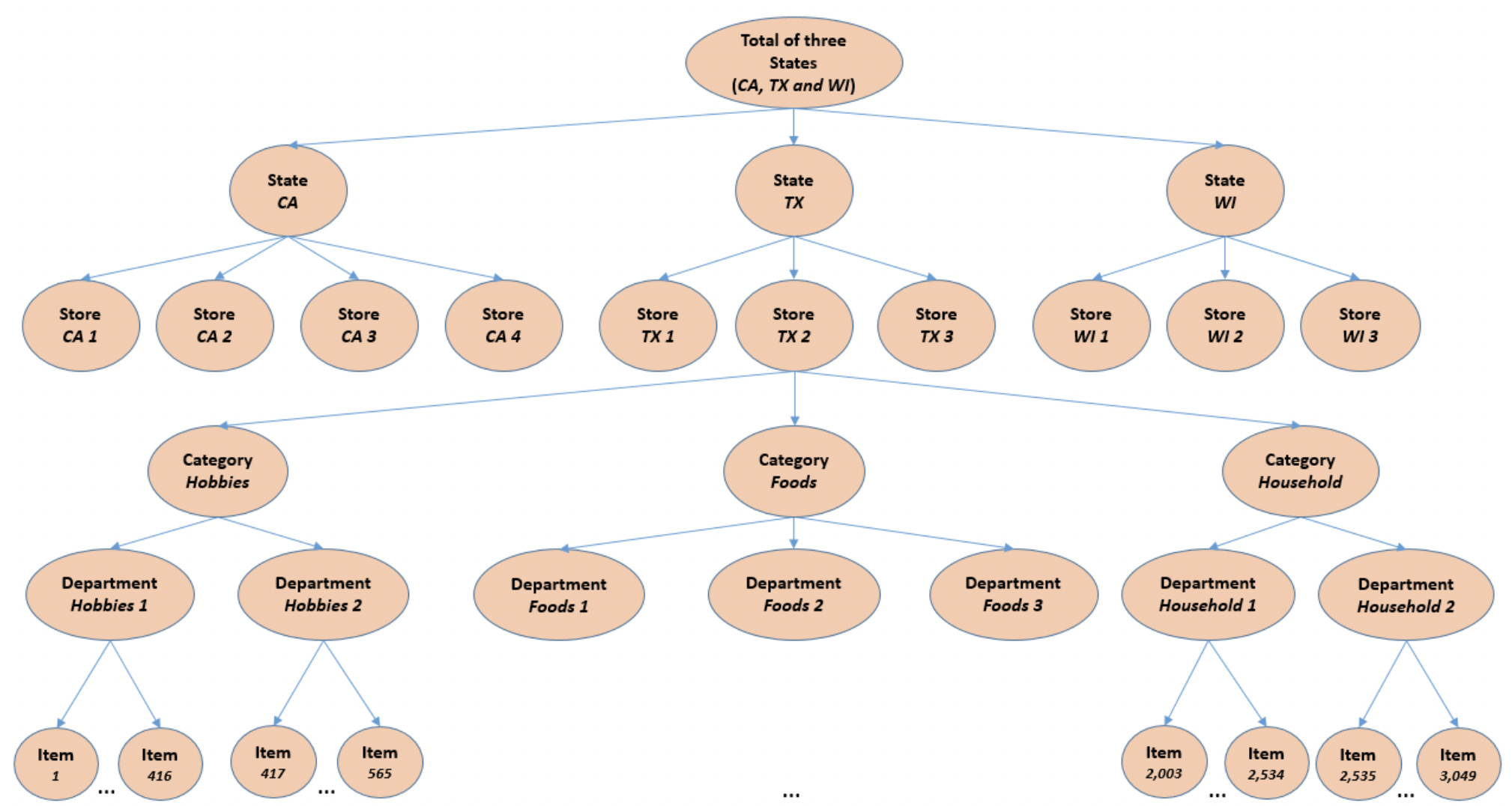}
    \caption{Hierarchical structure of the M5 dataset.}
    \label{fig:m5}
\end{figure*}

\subsection{M5 Competition Data}
The M5 dataset\footnote[2]{https://mofc.unic.ac.cy/wp-content/uploads/2020/03/M5-Competitors-Guide-Final-10-March-2020.docx} involves the unit sales of various products ranging from January 2011 to June 2016 in Walmart. It involves the unit sales of 3,049 products, classified into 3 product categories (Hobbies, Foods, and Household) and 7 product departments, where the categories mentioned above are disaggregated. The products are sold across ten stores in three states (CA, TX, and WI). An overview of how the M5 series are organized is shown in Figure \ref{fig:m5}. Here, we formulate a 4-level hierarchy, starting from the bottom-level individual item to unit sales of all products aggregated for each store.

\subsection{Wikipedia Webpage Views}
This dataset\footnote[3]{https://www.kaggle.com/c/web-traffic-time-series-forecasting} contains the number of daily views of 145k various Wikipedia articles ranging from July 2015 to Dec. 2016. We follow the data processing approach used in \citet{ben2019regularized} to sample 150 bottom-level series from the 145k series and aggregate to obtain the upper-level series. The aggregation features include the type of agent, type of access, and country codes. We then obtain a 5-level hierarchical structure with 150 bottom series.

\subsection{Australian Labour Force}
This dataset\footnote[4]{https://www.abs.gov.au/AUSSTATS/abs@.nsf\\/DetailsPage/6202.0Dec\%202019?OpenDocument} contains monthly employment information ranging from Feb. 1978 to Aug. 2019 with 500 records for each series. The original dataset provides a detailed hierarchical classification of labor force data, while we choose three aggregation features to formulate a 4-level symmetric structure. Specifically, the 32 bottom level series are hierarchically aggregated using labor force location, gender, and employment status. 

\begin{table*}[t!]
\centering
\caption{\centering{$\mathrm{MAPE}^{\downarrow}$ for small and large simulation dataset. The likelihood ratios are given in parentheses.}}
\hspace{-.85em}
\scalebox{1.0}{
\begin{tabular}{c|ccc|cccc}
\hlinewd{1.5pt}
\multirow{2}{*}{MAPE} & \multicolumn{3}{c}{ Simulation Small } & \multicolumn{4}{c}{ Simulation Large } \\
\hhline{|~|---|----|}
& Top level & Level 1 & Level 2 & Top level & Level 1 & Level 2 & Level 3\\
\hline
Base & 1.29 (.69) & 1.50 (.77) & 2.41 (.91) & 2.08 (.43) & 2.20 (.61) & 1.41 (.75) & 0.72 (.85) \\
BU & 2.14 (.73) & 1.76 (.79) & 2.41 (.91) & 4.19 (.46) & 3.48 (.64) & 1.48 (.76) & 0.72 (.85) \\
MinT-sam & 0.54 (.66) & 1.48 (.77) & 2.24 (.89) & 1.48 (.42) & 2.55 (.65) & 1.38 (.74) & 0.63 (.83) \\
MinT-shr & 0.45 (.65) & 1.47 (.77) & \textbf{2.23} (.89) & \textbf{1.28} (.39) & 2.31 (.63) & \textbf{1.35} (.74) & \textbf{0.59} (.81) \\
MinT-ols & \textbf{0.20} (.64) & 1.72 (.78) & 2.41 (.91) & 1.69 (.41) & 2.15 (.60) & 1.41 (.75) & 0.71 (.85) \\
ERM & 1.23 (.69) & 1.73 (.78) & 2.55 (.93) & 2.78 (.44) & 2.86 (.69) & 1.50 (.76) & 0.75 (.86) \\
SHARQ & 1.54 (.41) & \textbf{1.42} (.45) & 2.41 (.73) & 2.16 (.23) & \textbf{2.13} (.49) & 1.44 (.67) & 0.72 (.82) \\
\hlinewd{1.5pt}
\end{tabular}}
\label{sim1_result}
\end{table*}

\begin{table*}[t!]
\centering
\renewcommand\arraystretch{1.2}
\caption{\centering{$\mathrm{MAPE}^{\downarrow}$ on FTSE dataset. Level 1 is the top aggregation level; 4 is the bottom aggregation level.}}
\scalebox{0.76}{
\begin{tabular}{c|cccc|cccc|cccc|cccc}
\hlinewd{1.5pt}
Algorithm & \multicolumn{4}{c}{ RNN } & \multicolumn{4}{c}{ Autoregressive } & \multicolumn{4}{c}{ LST-Skip } & \multicolumn{4}{c}{ N-Beats } \\
\hline
\multirow{2}{*}{ Reconciliation } & \multicolumn{4}{c}{ Level } & \multicolumn{4}{c}{ Level } & \multicolumn{4}{c}{ Level } & \multicolumn{4}{c}{ Level } \\ \hhline{|~|----|----|----|----|}
& 1 & 2 & 3 & 4 & 1 & 2 & 3 & 4 & 1 & 2 & 3 & 4 & 1 & 2 & 3 & 4 \\ \hline
BU & 6.11 & 8.48 & 9.41 & 9.54 & 10.01 & 12.15 & 11.77 & 12.43 & 7.48 & 8.96 & 9.29 & 9.49 & 6.63 & 8.04 & 8.23 & 8.41 \\
Base & 4.82 & 6.27 & 8.55 & 9.54 & \textbf{8.65} & 10.46 & 10.88 & 12.43 & 6.02 & 7.79 & 8.76 & 9.49 & 5.86 & \textbf{7.56} & \textbf{8.01} & 8.41 \\
MinT-sam & 4.68 & 8.53 & 8.77 & 10.13 & 9.72 & 11.25 & 11.57 & 12.26 & 6.47 & 8.24 & 8.93 & 10.62 & 5.94 & 7.89 & 8.35 & 8.86 \\
MinT-shr & \textbf{4.43} & 8.46 & 8.59 & 9.75 & 9.23 & 10.91 & 11.02 & \textbf{12.13} & 6.12 & 8.11 & 8.81 & 10.57 & 5.67 & 7.74 & 8.22 & \textbf{8.54} \\
MinT-ols & 4.71 & 8.92 & 8.74 & 10.31 & 9.96 & 11.01 & 11.25 & 12.32 & 6.31 & 8.56 & 8.74 & 10.88 & 5.87 & 8.12 & 8.41 & 9.84 \\
ERM & 5.74 & 9.52 & 9.54 & 12.41 & 9.92 & 10.61 & 12.03 & 13.23 & 8.12 & 9.38 & 9.76 & 13.01 & 6.19 & 8.89 & 9.26 & 10.22 \\
SHARQ & 4.51 & \textbf{8.28} & \textbf{8.08} & \textbf{9.54} & 9.13 & \textbf{9.35} & \textbf{10.61} & 12.43 & \textbf{5.01} & \textbf{7.14} & \textbf{8.52} & \textbf{9.49} & \textbf{5.44} & 7.83 & 7.93 & 8.41 \\
\hlinewd{1.5pt}
\end{tabular}\hspace{5em}}
\label{tab:mape_ftse}
\end{table*}

\begin{table*}[t!]
\centering
\renewcommand\arraystretch{1.4}
\caption{\centering{$\mathrm{MAPE}^{\downarrow}$ on Wiki dataset. Level 1 is the top aggregation level; 5 is the bottom aggregation level.}}
\scalebox{0.56}{
\begin{tabular}{c|ccccc|ccccc|ccccc|ccccc}
\hlinewd{1.5pt}
Algorithm & \multicolumn{5}{c}{ RNN } & \multicolumn{5}{c}{ Autoregressive } & \multicolumn{5}{c}{ LST-Skip } & \multicolumn{5}{c}{ N-Beats } \\
\hline
\multirow{2}{*}{ Reconciliation } & \multicolumn{5}{c}{ Level } & \multicolumn{5}{c}{ Level } & \multicolumn{5}{c}{ Level } & \multicolumn{5}{c}{ Level } \\ \hhline{|~|-----|-----|-----|-----|}
& 1 & 2 & 3 & 4 & 5 & 1 & 2 & 3 & 4 & 5 & 1 & 2 & 3 & 4 & 5 & 1 & 2 & 3 & 4 & 5 \\ \hline
BU & 11.71 & 12.36 & 14.47 & 16.45 & 16.74 & 15.67 & 15.99 & 16.67 & 18.99 & 20.32 & 11.44 & 11.88 & 13.31 & 14.76 & 15.77 & 11.92 & 12.57 & 14.45 & 15.22 & 16.21 \\
Base & 11.12 & 11.52 & 14.06 & 16.11 & 16.74 & \textbf{15.04} & 15.23 & 16.02 & 17.83 & 20.32 & 11.21 & 11.24 & 12.88 & 14.35 & 15.77 & 11.84 & 12.02 & 14.17 & 15.16 & 16.21 \\
MinT-sam & 11.65 & 12.02 & 14.19 & 16.23 & 17.66 & 15.22 & 15.65 & 16.33 & 18.12 & 19.87 & 11.38 & 11.46 & 13.13 & 14.57 & 16.22 & 11.96 & 12.26 & 14.29 & 15.25 & 16.45 \\
MinT-shr & 11.32 & 11.86 & 13.87 & \textbf{16.07} & 17.54 & 15.17 & 15.12 & 15.98 & 17.69 & \textbf{19.54} & 11.24 & 11.15 & 12.91 & \textbf{14.32} & 16.14 & 11.75 & 12.19 & 14.03 & 15.02 & 16.39 \\
MinT-ols & 11.48 & 12.11 & 14.52 & 16.34 & 17.59 & 15.37 & 15.74 & 16.23 & 18.01 & 20.21 & 11.42 & 11.52 & 13.05 & 14.78 & 16.59 & 11.88 & 12.39 & 14.21 & 15.16 & 16.45 \\
ERM & 12.08 & 13.62 & 15.96 & 18.11 & 18.97 & 15.29 & 15.85 & 16.12 & \textbf{17.58} & 21.56 & 12.08 & 12.85 & 14.56 & 15.96 & 17.42 & 12.14 & 12.83 & 15.49 & 16.17 & 17.41 \\
SHARQ & \textbf{10.84} & \textbf{11.07} & \textbf{13.54} & 16.08 & \textbf{16.74} & 15.07 & \textbf{15.05} & \textbf{15.87} & 17.79 & 20.32 & \textbf{11.07} & \textbf{11.09} & \textbf{12.65} & 14.41 & \textbf{15.77} & \textbf{11.64} & \textbf{11.67} & \textbf{13.81} & \textbf{15.02} & \textbf{16.21} \\
\hlinewd{1.5pt}
\end{tabular}\hspace{5em}}
\label{tab:mape_wiki}
\end{table*}

\section{Additional Experiments} \label{sec:add_exp}
In this section, we demonstrate our additional experiment results, including the full results on FTSE and Wiki as well as additional simulation experiments under unbiasedness and Gaussian assumptions. Reconciliation error is also measured for each method. We start by discussing our evaluation metrics.

\subsection{Evaluation Metrics}
We denote $\hat{Y}_T (h)$ and $Y_T (h)$ as the $h-$step ahead forecast at time $T$ and its ground truth, respectively. To construct confidence intervals, we use the $95^{th}, 50^{th},$ and $5^{th}$ quantiles as upper, median and lower forecasts.

\begin{table*}[t!]
\centering
\caption{\centering{Average forecasting coherency on each dataset across 4 forecasting models. Bottom-level $\lambda = 3.0$, higher-level $\lambda$s are decreased gradually.}}
\hspace{-.65em}
\scalebox{1.0}{
\begin{tabular}{c|cccc}
\hlinewd{1.5pt}
\multirow{2}{*}{ Reconciliation } & \multicolumn{4}{c}{ Dataset } \\
\hhline{|~|----|}
& FTSE & Labour & M5 & Wiki \\
\hline
Base & 28.01 & 5.59 & 12.56 & 20.71 \\
BU & 0 & 0 & 0 & 0 \\
MINTsam & 4.21E-15 & 4.60E-12 & 0 & 5.46E-10 \\
MINTshr & 2.50E-15 & 4.19E-12 & 0 & 6.40E-11 \\
MINTols & 6.22E-15 & 6.10E-12 & 0 & 1.08E-10 \\
ERM & 6.48E-12 & 2.27E-08 & 5.86E-12 & 2.40E-07 \\
SHARQ & 1.59 & 0.53 & 0.22 & 2.63 \\
\hlinewd{1.5pt}
\end{tabular}}
\label{recon}
\end{table*}

\begin{table*}[t!]
\centering
\caption{\centering{Training and inference time (in second) comparison for each data set.}}
\scalebox{.9}{
\begin{tabular}{c|cc|cc|cc|cc}
\hlinewd{1.5pt}
\multirow{2}{*}{Time (s)} & \multicolumn{2}{c}{ FTSE } & \multicolumn{2}{c}{ Labour } & \multicolumn{2}{c}{ M5 } & \multicolumn{2}{c}{ Wikipedia } \\
\hhline{|~|--|--|--|--|}
& training & inference & training & inference & training & inference & training & inference \\
\hline
Base & 115.96 & 0.01 & 68.35 & 0.00 & 181.58 & 0.00 & 205.47 & 0.01 \\
BU & 65.83 & 0.03 & 57.06 & 0.00 & 105.45 & 0.00 & 142.53 & 0.01 \\
MinT-sam & 106.55 & 1,784.77 & 72.24 & 430.42 & 172.11 & 1,461.81 & 208.26 & 1,106.70 \\
MinT-shr & 104.35 & 1,148.49 & 60.83 & 317.02 & 175.83 & 1,039.53 & 198.16 & 788.31 \\
MinT-ols & 103.23 & 1,129.45 & 64.14 & 310.13 & 163.24 & 977.88 & 196.88 & 702.02 \\
ERM & 547.66 & 0.05 & 497.88 & 0.01 & 551.60 & 0.01 & 1,299.30 & 0.04 \\
SHARQ & 121.84 & 0.01 & 99.96 & 0.00 & 201.40 & 0.00 & 241.97 & 0.01 \\
\hlinewd{1.5pt}
\end{tabular}}
\label{time_comparison}
\end{table*}

\subsubsection{Mean Absolute Percentage Error (MAPE)}
The MAPE is commonly used to evaluate forecasting performance. It is defined by
\begin{equation}
    \mathrm{MAPE} = \frac{100}{H} \sum_{h=1}^H \frac{|Y_T (h) - \hat{Y}_T (h)|}{|Y_T (h)|}.
\end{equation}
\subsubsection{Likelihood Ratio}
We compute the likelihood ratio between the quantile prediction intervals versus the trivial predictors, which gives the specified quantile of training samples as forecasts. Specifically, define $N$ ($N=3$ in our case) as the number of quantile predictors. Then the likelihood ratio at $h-$step forecast is:
\begin{equation}
\alpha = \frac{\sum_{i=1}^N \rho_{\tau_i}(Y_T(h) - Q_{Y_T(h)}(\tau_i))}{\sum_{i=1}^N \rho_{\tau_i}(Y_T(h) - Q_{\mathcal{I}_T}(\tau_i))}.    
\end{equation}

Ideally, we should have $\alpha < 1$ if our estimator performs better than the trivial estimator.

\subsubsection{Continuous Ranked Probability Score (CRPS)}
CRPS measures the compatibility of a cumulative distribution function $F$ with an observation $x$ as:
\begin{equation}
    \mathrm{CRPS}(F, x) = \int_{\mathbb{R}} (F(z) - \mathbb{I}\{x \leq z\})^2 ~dz
\end{equation}
where $\mathbb{I}\{x \leq z\}$ is the indicator function which is one if $x\leq z$ and zero otherwise. Therefore, CRPS attains its minimum when the predictive distribution $F$ and the data are equal. We used this library\footnote[5]{https://github.com/TheClimateCorporation/properscoring} to compute CRPS.

\subsubsection{Reconciliation Error}
We compute the reconciliation error of forecasts generated by each method on each dataset to measure the forecasting coherency. More specifically, assume a total of $m$ vertices in the hierarchy at time $T$, the reconciliation error for the mean forecast is defined as 
\begin{equation}
    \frac{1}{H} \sum_{h=1}^H \sum_{i=1}^m \| \hat{Y}_T^i(h) - \sum_{e_{i, k} \in E} \hat{Y}_T^k (h)\|_1.
    \label{equ:re}
\end{equation}

\subsection{Simulation under Unbiased Assumption} 

We follow the data simulation mechanism developed in \citet{wickramasuriya2019optimal, ben2019regularized}, which satisfies the ideal unbiased base forecast and Gaussian error assumptions. The bottom level series were first generated through an ARIMA(p, d, q) process, where the coefficients are uniformly sampled from a predefined parameter space. The contemporaneous error covariance matrix is designed to introduce a positive error correlation among sibling series, while moderately positive correlation among others. We simulate a small and a large hierarchy with 4 and 160 bottom series, respectively. The bottom series are then aggregated to obtain the whole hierarchical time series in groups of two and four. For each series in the hierarchy, we generate 500 observations, and the final $h = 8, 16$ observations are used for evaluation for both the large and small hierarchies. We run the above simulation 100 times and report the average results. Table \ref{sim1_result} shows the average MAPE by fitting an ARIMA model followed by reconciliation on two simulation datasets. We can see that the MinT methods generally perform the best, particularly for MinT methods with shrinkage estimators. This confirms the statements from \citet{ben2019regularized, hyndman2011optimal} that under ideal unbiasedness assumption if the forecasting models are well specified, the MinT methods will provide the optimal solution. Simultaneously, the results of SHARQ are also satisfactory. In fact, it outperforms MinT methods at some levels.

\begin{table*}[h!]
\centering
\renewcommand\arraystretch{1.4}
\caption{\centering{Common Hyper-parameters for all experiments.}}
\scalebox{0.85}{
\begin{tabular}{c|c|c|c|c|c|c}
\hlinewd{1.5pt}
& Train/Valid/Test & Epoch & Learning Rate & Batch Size & Window Size & Horizon \\ \hline
Quantile Simulation & 0.6/0.2/0.2 & 300 & 1.00E-03 & 64 & 128 & 1 \\ 
Unbiased Simulation & 0.6/0.2/0.2 & 100 & 1.00E-03 & 128 & 10 & 1-8 \\
Real-world Data & 0.6/0.2/0.2 & 1000 & 0.1 & 128 & 168 & 1-8 \\
\hlinewd{1.5pt}
\end{tabular}}
\label{params}
\end{table*}

\subsection{Additional Results}
Table \ref{tab:mape_ftse} and \ref{tab:mape_wiki} show the MAPE results of FTSE and Wiki dataset. Moreover, table \ref{tab:lr} is the average likelihood ratio of each reconciliation method across four algorithms. The reported results are average across three random runs. We can see that SHARQ performs better overall in providing accurate probabilistic forecasts. Table \ref{time_comparison} compares the average training and inference time across all forecasting models. Overall, the training time of SHARQ and base forecast are roughly the same, but the inference time of SHARQ is ignorable relative to MinT, and ERM approaches. Since both these methods require matrix inversion to compute the weight matrix. Even if ERM could calculate the weight matrix on a separate validation set before inference, additional matrix computations are required to obtain the results.

\begin{table}[t!]
\centering
\renewcommand\arraystretch{1.4}
\caption{\centering{Average likelihood ratio across forecasting horizons and models.}}
\vspace{.5em}
\begin{tabular}{c|cccc}
\hlinewd{1.5pt}
Likelihood Ratio & Labour & M5 & FTSE & Wiki \\ \hline
BU & 0.36 & 0.48 & 0.50 & 0.66 \\
Base & 0.36 & 0.48 & 0.51 & 0.66 \\
MinT-sam & 0.36 & 0.47 & 0.50 & 0.66 \\
MinT-shr & 0.35 & 0.49 & 0.51 & 0.68 \\
MinT-ols & 0.34 & 0.48 & 0.51 & 0.66 \\
ERM & 0.35 & 0.48 & 0.51 & 0.67 \\
SHARQ & \textbf{0.07} & \textbf{0.25} & \textbf{0.32} & \textbf{0.65} \\
\hlinewd{1.5pt}
\end{tabular}\hspace{5em}
\label{tab:lr}
\end{table}

\subsection{Forecasting Coherency}
Table \ref{recon} compares the forecasting coherency of each reconciliation method. We use the metric defined in (\ref{equ:re}) to compute the forecasting reconciliation error generated by previous experiments. As expected, the MinT and ERM approach give almost perfect coherent forecasts, as these methods can directly compute the close form of weight matrix $P$ to optimally combine the original forecasts. Even though MinT and ERM can give perfectly coherent forecasts, the accuracy can sometimes be worse than the base method, which coincides with Proposition 2 (Hard Constraint). Although SHARQ could not give the most coherent results, there is still a significant improvement compared to incoherent base forecasts. Note that this can be further improved by increasing the penalty of the regularization term.

\section{Hyper-parameter Configurations}
We present hyper-parameters for all the experiments mentioned above. Table \ref{params} lists the common hyper-parameters used on each experiment. Model-specific hyper-parameters are as follows.

\paragraph{Quantile Simulation Experiment}
We simulate 500 samples for both step function and sinusoidal function; the data is trained on a vanilla RNN model with hidden dimension 5, layer dimension 2, and \textit{tanh} nonlinearity. We used 10 ensembles of estimators for bagging, and each model is trained using random 64 samples.

\paragraph{LSTNet}
The number of CNN hidden units: 100; the number of RNN hidden units: 100; kernel size of the CNN layers: 6; window size of the highway component: 24; gradient clipping: 10; dropout: 0.2; skip connection: 24. Note that to enable LSTNet to produce multi-quantile forecast, we add the final layer of each quantile estimator after the fully connected layer of the original model. The same linear bypass then adds the obtained quantile estimators to produce the final results.

\paragraph{N-Beats}
We use the same parameter settings as shown in the public GitHub repository\footnote[6]{https://github.com/philipperemy/n-beats}.




\end{document}